\documentclass{article}[11pt]

\usepackage{graphicx}

\usepackage{amsmath,amsfonts,amssymb,latexsym,epsfig}
\usepackage{mathrsfs}
\usepackage{verbatim}
\usepackage{latexsym}
\usepackage{amsthm}
\usepackage{amssymb}
\usepackage{subfig}
\usepackage{graphics}
\usepackage{amsbsy}
\usepackage{fullpage}
\usepackage{enumerate}
\usepackage{times}
\usepackage{multirow}
\usepackage{soul}
\usepackage[normalem]{ulem}

\usepackage{pgfplots}
\usepackage{tikz}
\usetikzlibrary{decorations.pathreplacing,decorations.markings,arrows}

\newcommand{\R}{\mathbb{R}}
\newcommand{\wA}{\widehat A}
\newcommand{\wB}{\widehat B}
\newcommand{\wC}{\widehat C}
\newcommand{\wD}{\widehat D}
\newcommand{\wL}{\widehat L}
\newcommand{\wM}{\widehat M}
\newcommand{\wN}{\widehat N}
\newcommand{\wP}{\widehat P}

\newcommand{\wS}{\widehat S}
\newcommand{\wZ}{\widehat Z}
\newcommand{\wsigma}{\widehat \sigma}
\newcommand{\calD}{\mathcal{D}}
\newcommand{\calE}{\mathcal{E}}
\newcommand{\calF}{\mathcal{F}}
\newcommand{\calG}{\mathcal{G}}
\newcommand{\calH}{\mathcal{H}}
\newcommand{\calT}{\mathcal{T}}
\newcommand{\E}{\mathbb{E}}

\newtheorem{theorem}{Theorem}[section]
\newtheorem{lemma}[theorem]{Lemma}
\newtheorem{example}[theorem]{Example}
\newtheorem{rem}[theorem]{Remark}

\newtheorem{proposition}[theorem]{Proposition}
\newtheorem{assumptions}[theorem]{Assumption}

\numberwithin{equation}{section}
\usepackage{amscd}
\usepackage{tikz}
\usepackage{xcolor}
\usepackage[english]{babel}
\usepackage[latin1]{inputenc}
\usepackage{times}
\usepackage[T1]{fontenc}
\usepackage{graphicx}
\usepackage{amsmath,amssymb}
\usepackage{slidesec}
\usepackage{verbatim}
\usepackage{url}
\usepackage{epsf}
\usepackage{amsmath}
\usepackage{graphics}
\usepackage{amsfonts}
\usepackage{amsbsy}
\usepackage{lscape}
\usepackage{enumerate}
\usepackage{amsthm}
\usepackage{amssymb}
\usepackage{exscale}
\usepackage{algorithm}
\usepackage{algorithmic}

\newcommand{\modk}[1]{\textcolor{black}{#1}\index {#1}}

\begin{document}
\title{Wasserstein distance estimates for the distributions of numerical approximations to  ergodic stochastic differential equations.
}
\author{J. M. Sanz-Serna $^{1}$ \and Konstantinos C. Zygalakis$^{2}$}

\maketitle

\begin{abstract}
We  present a  framework that allows for the non-asymptotic study of the \(2\)-Wasserstein distance between the invariant distribution of an ergodic stochastic differential equation and the distribution of its numerical approximation in the strongly log-concave case. This allows us to study in a unified way a number of different integrators proposed in the literature for  the overdamped and underdamped Langevin dynamics. In addition, we analyse a novel splitting method for the underdamped Langevin  dynamics which only requires one gradient evaluation per time step.
Under an additional smoothness assumption on a $d$--dimensional strongly log-concave distribution with condition number $\kappa$, the algorithm is shown to produce  with an $\mathcal{O}\big(\kappa^{5/4} d^{1/4}\epsilon^{-1/2} \big)$ complexity samples from a distribution that, in Wasserstein distance, is at most $\epsilon>0$ away  from the target distribution.
\end{abstract}
\footnotetext[1]{Departamento de Matem\'aticas,
	Universidad Carlos III de Madrid,
 	Legan\'es (Madrid), Spain}
\footnotetext[2]{School of Mathematics, University of Edinburgh, Edinburgh, Scotland}

\section{Introduction}
The problem of sampling from a target probability distribution $\pi^{\star}(x)$ in $\mathbb{R}^{d}$ is ubiquitous throughout applied mathematics, statistics, molecular dynamics, statistical physics and other fields. A typical approach for solving such problems is to construct a Markov process on $\mathbb{R}^{m}, \ m \geq d$ whose equilibrium distribution (or a suitable marginal of it) is designed to coincide with $\pi^{\star}$ \cite{MCMChand}. Often such Markov processes are obtained by solving stochastic differential equations (SDEs). Two typical examples of such SDEs are  the overdamped Langevin equation, {\color{black} $c>0$,}
\begin{equation}\label{eq:overdamped}
dx = -c \nabla f(x)\,dt+\sqrt{2c}\, dW(t),
\end{equation}
and the underdamped Langevin equation, {\color{black} $c,\gamma>0$,}
\begin{subequations}\label{eq:underdamped}
\begin{eqnarray}
\label{eq:underdampedv} dv & =& -\gamma v\,dt-c \nabla f(x)\,dt +\sqrt{2\gamma c}\, dW(t),\\
\label{eq:underdampedx} d x &=& v\,dt.
\end{eqnarray}
\end{subequations}
Under mild assumptions on $f(x)$ one can show that the dynamics of \eqref{eq:overdamped} is ergodic with respect to the distribution $\pi^{\star}$ with density $\propto \exp(-f(x))$,  while the dynamics of \eqref{eq:underdamped} is ergodic with respect to $\pi^{\star}$ with density  $\propto \exp(-f(x){\color{black} -}\frac{1}{2c}\|v\|^{2})$.

Equations \eqref{eq:overdamped}, \eqref{eq:underdamped} \cite{SS14} provide the basis for
different computational devises for sampling from $\pi^{\star}$. In
particular, one can obtain samples from $\pi^{\star}$ by discretizing the SDEs and generating numerical solutions over a long time  interval \cite{MT07}. One needs to be careful with the integrator that is used, since it could be the case that the discrete Markov chain resulting from the numerical discretization might not be ergodic \cite{RT96}. In addition,  even if that chain is ergodic, it is normally the case that the stationary distribution  $\widehat{\pi}^{\star}$ of the numerical solution  is  different from $\pi^{\star}$.  The study of the asymptotic error between $\widehat{\pi}^{\star}$ and $\pi^\star$ has received a lot of attention in the literature. The work in \cite{MST10} investigates the effect of the numerical discretization on the convergence of  the ergodic averages, while \cite{AGZ14} present general order conditions for  \modk{the numerical invariant measure $\widehat{\pi}^{\star}$  to approximate $\pi^{\star}$
with high order of accuracy,}  by exploiting the connections between partial differential equations and SDEs \cite{LS16}. A number of recent papers have applied this framework to numerical integrators for the underdamped Langevin equation \cite{AGZ15}, as well as to the case of stochastic gradient Langevin dynamics \cite{VZT16,LS16a}. In addition,  \cite{GV15,LV20a} extended this framework to the case of post processed integrators and to  SDEs on manifolds.

Another line of research that has received much attention in the
last few years deals with the study of the non-asymptotic
error between the numerical approximation and the invariant measure $
\pi^\star$. In particular, for the case of the overdamped Langevin equation
\eqref{eq:overdamped} and log-concave and strongly log-concave
distributions \cite{D17b} established non-asymptotic bounds in total
variation distance for the Euler-Maruyama method and an explicit extension of it
based on further smoothness assumptions. These results have been extended
to the Wasserstein distance \(W_2\) in e.g.\ \cite{DM17,D17a,DM19,DK19,DMM18}, while the paper
\cite{HSR19} obtains similar bounds for implicit methods applied to
\eqref{eq:overdamped}. Similar non-asymptotic analyses for the case of
the underdamped Langevin equation appear in \cite{CCB18,DD20,M20,ShenLee2019,FLO21}. One of the aims of all that literature is to study the number of steps \(n\) that the integrators require to achieve a target accuracy \(\epsilon\) when applied to \(d\)-dimensional targets with condition number \(\kappa\). Underdamped discretizations may lead to a better dependence of \(n\) on \(\epsilon\) and \(d\) than
their overdamped counterparts.
The case of non-strongly log-concave distributions  and the non-asymptotic behaviour of numerical algorithms has also  received   attention recently \cite{DKL19,MMS20}.

In this work, we  present a unified framework that
allows for the non-asymptotic study of numerical methods for ergodic
stochastic differential equations (including equations
\eqref{eq:overdamped} and \eqref{eq:underdamped}) in the case of
strongly log-concave distributions. In particular, we obtain a general
bound for the error in $W_{2}$ between $\pi^{\star}$ and the probability
distribution of the numerical solution after $n$-iterations. This bound
depends on two factors,   the first can be controlled by understanding the
contractivity properties of the numerical method, while the second is
directly related to the local strong {\color{black} error} of the integrator. Moving to integrators with smaller strong local error results in a better performance when the dimensionality grows and the error level \(\epsilon\) decreases. Also moving to integrators that are contractive for  larger step sizes improves the performance for large condition numbers.
This is consistent with what has been suggested in the literature
\cite{HSR19,MVZ20}.

As an application of the suggested framework, we study two numerical methods for the underdamped Langevin dynamics. The first  is the method, {\color{black} that we shall call EE,} used in \cite{CCB18}; the second is a splitting method called UBU. Both require the same computational effort, {\color{black} namely one gradient evaluation per time-step,} but UBU has better convergence properties \cite{AS16,AS19}.

{\color{black}
\begin{itemize}
\item For the integrator EE, we prove that, in \(2\)-Wasserstein distance and for a strongly log-concave \(d\)-dimensional distribution with condition number \(\kappa\), the algorithm  produces a distribution that  is \(\epsilon\)--away from the target in a number of steps that (up to logarithmic terms) behaves like $\mathcal{O}(\epsilon^{-1}\kappa^{3/2} d^{1/2})$. This improves on the  $\mathcal{O}(\epsilon^{-1}\kappa^{2} d^{1/2})$
estimate in \cite{CCB18}. EE has also been analysed in \cite{DD20}; however, the analysis in that reference has severe limitations as discussed in Section~\ref{sec:applunderdamped}.
 \item UBU, under the same hypotheses as EE, shares the  $\mathcal{O}(\epsilon^{-1}\kappa^{3/2} d^{1/2})$ estimate. However, unlike EE, UBU is capable of leveraging
additional smoothness properties of the log-density of the target. With such an additional smoothness assumption, we prove an estimate that depends on \(\epsilon\), \(\kappa\) and \(d\) as $\mathcal{O}(\epsilon^{-1/2}\kappa^{5/4} d^{1/4})$ (there is also a dependence on  a bound for the third derivatives of the target log-density).
\end{itemize}

 Even though a detailed comparison between UBU and alternative algorithms is not within the scope of the present paper, the following comments are in order.

\begin{itemize}
\item As we will discuss in Remark~\ref{rem:secondorder}, for fixed \(\kappa\), the improvement
from the $\epsilon^{-1}d^{1/2}$ EE estimate to the $\epsilon^{-1/2}d^{1/4}$ UBU estimate
arises from EE having strong order one and UBU having strong order two. This shows the importance of strong second-order integrators. A strong second-order discretization of the underdamped Langevin dynamics that requires evaluation of the Hessian has been introduced in \cite{DD20}. However the analysis in that reference only holds for unrealistic values of the stepsize, see Section~\ref{ss:ubujuly}.
\item The randomized midpoint method in \cite{ShenLee2019} uses \emph{two} gradient evaluations per
time-step and may be regarded as optimal \cite{CaoLuWang}  among the integrators of the underdamped Langevin dynamics that  use the driving Brownian motion, its weighted integration and target and an oracle of the $\nabla f$.
For Lipschitz gradients, the estimate of the mixing time
is \(\mathcal{O}(\epsilon^{-1/3} \kappa^{7/6}d^{1/6}+\epsilon^{-2/3} \kappa d^{1/3})\), where we note the favourable dependence on $\kappa$, which stems from the random nature of the algorithm (see \cite{CaoLuWang} and its references). For fixed $\kappa$ we then find an $\epsilon^{-2/3}  d^{1/3}$ behaviour, to be compared with the $\epsilon^{-1/2}  d^{1/4}$
estimate of UBU when the extra smoothness assumption holds. See Remark~\ref{rem:secondorder}.
\item The algorithm suggested in \cite{MMW21} is \emph{not} based on integrating the underdamped Langevin equation but an alternative system of stochastic differential equations where $x$ has additional smoothness (see Remark~\ref{rem:distinct}).
For fixed \(\kappa\), the authors prove a  \(\mathcal{O}(\epsilon^{-1/2}d^{1/4})\) estimate of the mixing time (i.e. the behaviour established here for UBU).
That reference does not investigate the dependence of the mixing time on \(\kappa\); numerical experiments suggest the algorithm does not operate satisfactorily when the condition number is large.
\end{itemize}
}

The main contributions of this work are:
\begin{enumerate}
\item The use of an appropriate state-form representation of SDEs and \modk{their}   numerical integrators that allows to establish contractivity estimates both for the time-continuous process and \modk{its} numerical solution.
\item A study of the contractivity of integrators for  the underdamped Langevin dynamics that takes into account the possible impact of increasing condition numbers.
\item A general result that allows to obtain bounds for the \(2\)-Wasserstein distance between the target distribution and its numerical approximations for general SDEs. In particular the result may be applied to discretizations of the overdamped and underdamped Langevin equations.
\item We improve on the analysis in  \cite{CCB18} and explain the reasons why similar improvements may be expected when analysing other integrators.
\item We suggest the use in sampling of  UBU, a splitting integrator for the underdamped Langevin equations that only requires one gradient evaluation per step and possesses second order weak and strong accuracy.
\item We provide non-asymptotic estimates of the sampling accuracy of UBU.
\end{enumerate}

The rest of the paper is organised as follows. In Section \ref{sec:prel} we set up notation and discuss the different smoothness assumption on $f$ that we will employ through out the paper. In Section~\ref{sec:sdes}  we present the stochastic differential equations (SDEs) we are concerned with. These are written in a state-space form framework, similar to that used (for other purposes) in \cite{LRP16,FRMP,SSKZ20a}. This framework is useful here because it makes it easy (see Propositions~\ref{prop:contrac}  and \ref{prop:LZ}) to investigate the contractivity properties that underlie the SDE Wassertein distance estimates between the push-forward in time of two initial probability distribution (Proposition~\ref{prop:twice}). Section~\ref{sec:integrators}, parallel to Section~\ref{sec:sdes}, is devoted to the integrators and their contractivity.
Again a  state-space framework is used that makes it possible to easily investigate the contractivity of the integrators.
Section~\ref{sec:main} contains one of the main contributions of this paper, Theorem~\ref{theo:main}, which provides a general result for getting bounds of the Wasserstein distance between the invariant distribution \(\pi^\star\) of the SDE  and the distribution of the numerical solution. To apply Theorem~\ref{theo:main} one needs (1) to establish a contractivity estimate \emph{for the integrator} and (2) to prove what we call a local error bound. The latter is essentially a mean square bound of the difference between a single step of the integrator and a corresponding step with the SDE, under the assumption that the initial data for the step follows the distribution \(\pi^\star\).
Section~\ref{sec:applunderdamped} applies the general result to investigate two discretizations of the underdamped Langevin dynamics. The final Section~\ref{sec:final} contains some additional results and also the more technical proofs of the results in the preceding sections.

The extension of the material in this paper to variable step sizes and to inaccurate gradients is certainly possible, but will not be considered.


\section{Preliminaries} \label{sec:prel}
We will now discuss some assumptions on $f$, as well as set up some notation that we will later use.

\subsection{Smoothness properties of $f$}

The symbol \(\|\cdot\|\) always refers to the standard Euclidean norm.
Throughout the paper we shall assume that the following two conditions hold:
\begin{assumptions} \label{as1}
$f: \mathbb{R}^{d} \rightarrow \mathbb{R}^{d}$ is twice differentiable  and $L$-smooth, \emph{i.e}
\begin{equation}
\forall x,y \in \mathbb{R}^{d}, \qquad\|\nabla f(x)- \nabla f(y)\| \leq L \|x-y\|.
\end{equation}
\end{assumptions}

\begin{assumptions} \label{as2}
$f: \mathbb{R}^{d} \rightarrow \mathbb{R}^{d}$ is $m$-strongly convex,  \emph{i.e}
\[
\forall x,y \in \mathbb{R}^{d}, \qquad f(y) \geq  f(x) + \left \langle \nabla f(x),y-x \right \rangle +\frac{m}{2} \|x-y\|^{2}.
\]
\end{assumptions}

It is well known that these two assumptions  are equivalent to the Hessian of $f$, which we will denote by $\calH:\mathbb{R}^{d} \rightarrow \mathbb{R}^{d \times d}$, being positive definite and satisfying $m I_{d\times d} \preceq \calH(x) \preceq L I_{d\times d}$. In studies like the present one, Assumptions \ref{as1} and \ref{as2} are standard in the literature: see, among others, \cite{D17a,DM17,DM19,DK19} for the overdamped Langevin dynamics and \cite{CCB18,DD20,M20,FLO21} for the underdamped case.

In addition to these two assumptions, the following further smoothness assumption on $f$ will be used when it comes to analysing higher  strong-order discretizations for the underdamped Langevin equation. The symbol \(\calH^\prime\) denotes the tensor of third derivatives (derivative of the Hessian); at each \(x\in\R^d\),  \(\calH^\prime(x)\) is a bilinear operator mapping pairs \((w_1,w_2)\in \R^d\times\R^d\) into vectors in \(\R^d\).
\begin{assumptions} \label{as3}
 \(f\) is three times  differentiable and there is a constant \(L_1\geq 0\) such that at each point \(x\in\R^d\), for arbitrary \(w_1\), \(w_2\):
\[
\| \calH^\prime(x) [w_1,w_2]\| \leq L_1 \|w_1\|\,\|w_2\|
\]
\end{assumptions}
\subsection{Wasserstein distance} Let $\pi_{1}$ and $\pi_{2}$ be two probability measures on $\R^{N}$. The $2$-Wasserstein distance between $\pi_{1}, \pi_{2}$ is given by
\[
W_2(\pi_1,\pi_2) = \left(\inf_{\zeta\in Z} \int_{\R^N} \| x-y \|^2 d\zeta(x,y)\right)^{1/2},
\]
where \(Z\) is the set of all couplings \cite{EGZ19}  between \(\pi_1\) and \(\pi_2\), i.e.\ the set of all probability distributions in \(\R^N\times \R^N\) whose marginals on the first and second factors are \(\pi_1\) and \(\pi_2\) respectively. More generally, if \(P\) is an \(N\times N\) positive definite symmetric matrix, we will use the distance
\[
W_P(\pi_1,\pi_2) = \left(\inf_{\zeta\in Z} \int_{\R^N} \| x-y \|_P^2 d\zeta(x,y)\right)^{1/2},
\]
where in the \(P\)-norm defined by \(\|\xi\|_P = (\xi^TP\xi)^{1/2}\).
Since the \(P\)-norm and the standard Euclidean norm are related by
\begin{equation}\label{eq:sandwich}
p_{\min} \|\cdot\|^2 \leq \|\cdot\|_P^2 \leq p_{\max} \|\cdot\|^2,
\end{equation}
where \(p_{\min}\) and \(p_{\max}\) are the smallest and largest eigenvalues of \(P\), we also have
\[
p_{\min} W_2^2(\pi_1,\pi_2) \leq W_P^2(\pi_1,\pi_2) \leq p_{\max} W_2^2(\pi_1,\pi_2),
\]
for arbitrary \(\pi_1\), \(\pi_2\). Therefore bounds for the metric \(W_P\) may immediately be translated into bounds
for \(W_2\) and viceversa.

\section{Stochastic differential equations}
\label{sec:sdes}
In this section we will study some properties of a class of  ergodic stochastic differential equations that includes \eqref{eq:overdamped} and \eqref{eq:underdamped}. In particular, we will extend to the stochastic case a control theoretical framework used in \cite{FRMP,LRP16} to analyse optimization algorithms, and study properties of such SDEs, including the existence  of an invariant measure, and the speed of convergence to equilibrium in the Wasserstein distance.

\subsection{State-space form}

{\color{black} We are concerned with sampling algorithms obtained by discretizing  SDEs  with additive noise that may be written as linear systems in state-space form:}\footnote{\color{black} Note that this excludes algorithms, like the Riemann manifold MALA in \cite{girolamicalderhead}, that use multiplicative noise. Also Hamiltonian Montecarlo \cite{BoSS18} and similar piecewise deterministic samplers that use jumps do not fit in the present study.}
\begin{subequations} \label{eq:cont}
\begin{eqnarray}
 \label{eq:sys1bis} d \xi(t) &=&  A\xi(t)dt+ Bu(t)dt+ \sigma dW(t), \\
 \label{eq:sys2bis} x(t) &=&  C\xi(t), \\
 \label{eq:sys3bis} u(t) &=& \nabla f(x(t)).
\end{eqnarray}
\end{subequations}
Here \(\xi\in \R^N\) is the state, \(u\in \R^d\) is the input, \(x\in\R^d\) is the output that is mapped to \(u\) by the nonlinear map \(\nabla f\) and \(W\) represents the standard \(M\)-dimensional Brownian motion. The real matrices \(A\), \(B\), \(C\) and \(\sigma\) are constant, with sizes \(N\times N\), \(N\times d\), \(d\times N\) and \(N\times M\) respectively. We define
\[D = (1/2)\sigma\sigma^T.\] and note that, since the right hand-side of \eqref{eq:sys1bis} is globally Lipschitz continuous, the solution exists and is unique.

\begin{example}The simplest case corresponds to the overdamped Langevin equation \eqref{eq:overdamped}
{\color{black} (the positive constant $c$}
 may be set \(=1\) by rescaling \(t\)) and \(W\)  \(d\)-dimensional.
 {\color{black} Here,} \(N = d\), \(M = d\), \(\xi = x\), \(A = 0_{d\times d}\), \(B = -cI_d\), \(C =I_d\), \(\sigma = \sqrt{2c}I_d\), \(D = cI_d\).
\end{example}

\begin{example}The underdamped Langevin dynamics \eqref{eq:underdamped}
(\(\gamma\) and \(c\) are positive constants and \(W\) is \(d\)-dimensional) has \(N = 2d\), \(M = d\), \(\xi = [v^T,x^T]^T\), and (\(0\) stands for \(0_{d\times d}\))
\[
A = \left[\begin{matrix}-\gamma I_d & 0\\ I_d & 0\end{matrix}\right],\quad
B = \left[\begin{matrix}-c I_d \\  0\end{matrix}\right],\quad
C = \left[\begin{matrix}0& I_d\end{matrix}\right],\quad
\sigma = \left[\begin{matrix}\sqrt{2\gamma c} I_d \\  0\end{matrix}\right],\quad
D = \left[\begin{matrix}\gamma c I_d & 0\\ 0 & 0\end{matrix}\right].
\]
\end{example}

\begin{rem}\label{rem:distinct} As distinct from the situation in \eqref{eq:overdamped}, in \eqref{eq:underdamped} the noise \(W(t)\) does not enter the \(x\) equation directly; it does so only through the auxiliary variable \(v\). This results in \(x(t)\) being smoother in the underdamped case than in the overdamped case. This idea may be taken further: additional  auxiliary variables may be introduced so as to increase the smoothness of \(x(t)\), see e.g.\ \cite{MMW21}.
\end{rem}

The following proposition, whose proof is given in Section~\ref{sec:proofprop1}, relates \eqref{eq:cont} and \modk{the} pdf \(\propto\exp\big(-f(x)\big)\). The proposition may be used to check that the target  is in fact \modk{the} invariant density for  the overdamped Langevin dynamics \eqref{eq:overdamped} and that the underdamped Langevin system \eqref{eq:underdamped} has the invariant density \(\propto\exp\big(-f(x)-\|v\|^2/(2c)\big)\).

\begin{proposition}\label{prop:connection}Assume that \(S\) is an \(N\times N\) positive semidefinite symmetric matrix.

\begin{itemize}
\item The relations
\begin{subequations} \label{eq:relations}
\begin{eqnarray}
{\rm Tr}(A+DS) &=& 0,\\
C B+ C D C^T &=& 0,\\
CA+B^TS+2 C D S & =& 0,\\
S A+A^TS+2S DS& = &0,
\end{eqnarray}
\end{subequations}
 imply that \eqref{eq:cont} has the invariant probability distribution \(\pi^\star\) with density \[\propto \exp\big(-f(C\xi){\color{black}-}(1/2) \xi^T S\xi\big).\]

\item If \(SC^T =0\), then the marginal of \(\propto \exp\big(-f(C\xi)-(1/2) \xi^T S\xi\big)\) on \(x=C\xi\) is the target
\(\propto \exp(-f(x))\).
\end{itemize}
\end{proposition}
If \(f\) is regarded as being arbitrary, then the relations \eqref{eq:relations} are also necessary for the probability distribution with density \(\propto \exp\big(-f(C\xi){\color{black}-}(1/2) \xi^T S\xi\big)\) to be invariant, see Section~\ref{sec:proofprop1}.
The next result may be useful to check the hypotheses of Proposition~\ref{prop:connection}. The proof is a simple exercise and will not be given.

\begin{proposition}\label{prop:connectiontwo}
The relations \eqref{eq:relations} hold if
\[A = -(D+R) S, \qquad B = -(D+R) C^T,
\]
where \(R\) is an arbitrary \(N\times N\) skew-symmetric matrix.
\end{proposition}

\subsection{Convergence to the invariant distribution}
We assume hereafter that \eqref{eq:cont} has the \modk{unique} invariant distribution \(\pi^\star\).
If \(\pi\) denotes the probability distribution of the initial value \(\xi(0)\) for \eqref{eq:cont} and \(\Phi_t\pi\), \(t\geq 0\) represents the resulting  probability distribution of \(\xi(t)\), we will investigate the
convergence,  in the Wasserstein distance, of \(\Phi_t\pi\) towards  \(\pi^\star\), as \(t\rightarrow\infty\).

In order to estimate \(W_P(\Phi_t\pi_1, \Phi_t \pi_2)\)  we use the following well-known approach. We introduce the
auxiliary \(2N\)-dimensional SDE:
\begin{subequations} \label{eq:twice}
\begin{eqnarray}
d \xi^{(1)}(t) &=&  A\xi^{(1)}(t)dt+ B\nabla f(C\xi^{(1)}(t))dt+ \sigma dW(t),\\\qquad d \xi^{(2)}(t) &=&  A\xi^{(2)}(t)dt+ B\nabla f(C\xi^{(2)}(t))dt+ \sigma dW(t),
\end{eqnarray}
\end{subequations}
where the \emph{same}  Brownian motion $W(t)$ drives \(\xi^{(1)}(t)\) and \(\xi^{(2)}(t)\). If  \(\xi^{(1)}(0) \sim \pi_1\) and
\(\xi^{(2)}(0) \sim \pi_2\), and  \(\zeta\) is a coupling between \(\pi_1\) and \(\pi_2\) then the pushforward of \(\zeta\) by the solution of \eqref{eq:twice} provides a coupling for the distributions \(\Phi_t\pi_1\) and  \(\Phi_t\pi_2\)  of \(\xi^{(1)}(t)\) and \(\xi^{(2)}(t)\). In this setting it is easy to prove the following result.

\begin{proposition}\label{prop:twice}
Assume that \(P\succ 0\) and \(\lambda>0\) exist such that for \eqref{eq:twice}, almost surely,
\begin{equation}\label{eq:contractP}
\|\xi^{(2)}(t) -\xi^{(1)}(t)\|_P^2 \leq e^{-\lambda t} \|\xi^{(2)}(0) -\xi^{(1)}(0)\|_P^2,\qquad t>0.
\end{equation}
Then, for arbitrary distributions, \(\pi_1\) and \(\pi_2\),
\[
W_P(\Phi_t\pi_1,\Phi_t\pi_2) \leq e^{-\lambda t/2}W_P(\pi_1,\pi_2),\qquad t>0,
\]
and, in particular, for arbitrary \(\pi\),
\begin{equation}\label{eq:contractwasser}
W_P(\Phi_t\pi,\pi^\star) \leq e^{-\lambda t/2}W_P(\pi,\pi^\star),\qquad t>0.
\end{equation}
\end{proposition}

\subsection{Contractivity}

We now identify sufficient conditions for \eqref{eq:contractP} to hold.

\begin{lemma}\label{lem:withu}
Let \(P\succ 0\) be an \(N\times N\) symmetric matrix and \(\lambda>0\). For solutions of \eqref{eq:twice},
\begin{align*}
&d \Big(e^{\lambda t} [\xi^{(2)}(t)-\xi^{(1)}(t)]^T P [\xi^{(2)}(t)-\xi^{(1)}(t)]\Big)  =
\\
&\qquad\qquad e^{\lambda t} \Big( [\xi^{(2)}(t)-\xi^{(1)}(t)]^T (\lambda P+A^TP+PA)[\xi^{(2)}(t)-\xi^{(1)}(t)]\\
&\qquad\qquad  +[u^{(2)}(t)-u^{(1)}(t)] B^TP [\xi^{(2)}(t)-\xi^{(1)}(t)]\\ & \qquad\qquad+[\xi^{(2)}(t)-\xi^{(1)}(t)]^TPB[u^{(2)}(t)-u^{(1)}(t)]\Big)\:dt.
\end{align*}
\end{lemma}
\begin{proof} It is enough to apply Ito's rule to \[ F(t,\xi^{(1)}(t),\xi^{(2)}(t)) = e^{\lambda t} [\xi^{(2)}(t)-\xi^{(1)}(t)]^T P [\xi^{(2)}(t)-\xi^{(1)}(t)];\]
the Ito correction is
\[
{\rm Tr}\left(
\left[\begin{matrix}\sigma^T&\sigma^T\end{matrix}\right]
 \left[\begin{matrix}P&-P\\ -P&P \end{matrix}\right]
 \left[\begin{matrix}\sigma\\ \sigma \end{matrix}\right]
 \right)=0.
\]
\end{proof}

The inputs \(u^{(1)}(t)\), \(u^{(2)}(t)\) that appear in the lemma may be eliminated by using that \(\nabla f(x)\) is continuously differentiable. In fact, by
the mean value theorem,
\begin{eqnarray*}
u^{(2)}(t)-u^{(1)}(t) &= &\bar\calH(x^{(2)}(t),x^{(1)}(t))\, [x^{(2)}(t)-x^{(1)}(t)] \\&=&\bar\calH(x^{(2)}(t),x^{(1)}(t)) C\, [\xi^{(2)}(t)-\xi^{(1)}(t)],
\end{eqnarray*}
 where, for each pair of vectors \(y_1\), \(y_2\) in \(\R^d\), we have defined
\[
\bar\calH(y_2,y_1) =\int_0^1 \calH\big(y_1+z [y_2-y_1]\big)\, dz
\]
(\(\calH\) is the Hessian of \(f\)). After elimination of the inputs,  \modk{Lemma \ref{lem:withu}} yields
\begin{align*}
&d \Big(e^{\lambda t} [\xi^{(2)}(t)-\xi^{(1)}(t)]^T P [\xi^{(2)}(t)-\xi^{(1)}(t)]\Big)  = \\
&\qquad e^{\lambda t} [\xi^{(2)}(t)-\xi^{(1)}(t)]^T\\ &\qquad \Big(\lambda P + P\big(A+B\bar\calH(x^{(2)}(t),x^{(1)}(t)) C\big)+\big(A+B\bar\calH(x^{(2)}(t),x^{(1)}(t)) C\big)^TP\Big)
 \\&\qquad [\xi^{(2)}(t)-\xi^{(1)}(t)]\:dt,
\end{align*}
an equality that implies our next result.

\begin{proposition}\label{prop:contrac}
Let \(P\succ 0\) be an \(N\times N\) symmetric matrix and \(\lambda >0\).
Assume that, for each \(y_1,y_2\in\R^d\), the matrix
\[
\calT(\lambda,P,y_1,y_2) = \lambda P + P\big(A+B\bar\calH(y_1,y_2) C\big)+\big(A+B\bar\calH(y_1,y_2) C\big)^TP
\]
is \(\preceq 0\).  Then,  for solutions of \eqref{eq:twice} the contractivity estimate \eqref{eq:contractP} holds almost surely.
\end{proposition}

\subsection{Checking contractivity}
\label{sec:checkincontractivity}
We  next provide a result that is useful when checking the hypothesis \(\calT\preceq 0\) in the last proposition.

 Typically, in \eqref{eq:cont}
\begin{equation} \label{eq:hatmatrices}
A = \wA\otimes I_d,\qquad B= \wB \otimes I_d,\qquad C = \wC\otimes I_d,
\end{equation}
with \(\wA\), \(\wB\), and \(\wC\)  of sizes \(\wN\times \wN\), \(\wN\times 1\), and \(1\times \wN\)  respectively (which implies that
\(N = \wN d\)). This is for instance the situation for the overdamped and underdamped Langevin equations presented above, {\color{black} where $\wN=1$ and $\wN=2$ respectively. In general $\wN$ will be a small integer and therefore the matrices \(\wA\), \(\wB\), and \(\wC\) will also be small.}

When \eqref{eq:hatmatrices} holds and also \(\sigma = \wsigma\otimes I_d,\) (with \(\wsigma\) of size \(\wN\times \wM\)) and \(S=\wS\otimes I_d\),  the hypotheses of Proposition~\ref{prop:connection} may be stated in terms of the matrices with a hat, i.e.,  in the second item, \(\wS\wC^T= 0\) and,
in the first item, \({\rm Tr}(\wA+\wD\wS) = 0\), etc. (here \(\wD = (1/2) \wsigma\wsigma^T\)). The same observation applies to Proposition~\ref{prop:connectiontwo}.
In addition, it makes sense to consider that the matrix \(P\succ 0\) is of the form  \(\wP\otimes I_d\) with \(\wP\) of size \(\wN\times \wN\). Note that the eigenvalues of \(P\) are obtained by repeating \(d\) times each eigenvalue of \(\wP\) and in paticular \(P\succ 0\) if and only if \(\wP\succ 0\). We then have:

\begin{lemma}\label{lemma:con2} Assume that \eqref{eq:hatmatrices} holds and \(P=\wP\otimes I_d\). The set of the \(N = \wN d\) eigenvalues of \(\calT(\lambda,P,y_1,y_2)\) is the union of the sets of eigenvalues of the  matrices (of size \(\wN\times \wN\))
\begin{equation}\label{eq:Hmatrix}
\lambda \wP + \wP\big(\wA+H_i(y_1,y_2)\wB \wC\big)+\big(\wA+H_i(y_1,y_2)\wB \wC\big)^T\wP
\end{equation}
where \(H_i(y_1,y_2)\), \(i=1,\dots, d\), are the eigenvalues of \(\bar\calH(y_1,y_2)\).
\end{lemma}
\begin{proof} After using \eqref{eq:hatmatrices} and \(\bar\calH = 1\otimes\bar\calH\), the mixed product property of \(\otimes\) implies:
\[
\calT = (\lambda \wP +\wP\wA+\wA^T\wP)\otimes I_d+ (\wP\wB \wC+\wC^T \wB^T\wP)\otimes\bar\calH
\]
Now factorize \(\bar\calH = Q \calD Q^T\) with \(\calD\) diagonal and \(Q\) orthogonal (both \(d\times d\)). It follows that
\[
\calT = (I_{\wN} \otimes Q) \Big[(\lambda \wP +\wP\wA+\wA^T\wP)\otimes I_d+ (\wP\wB \wC+\wC^T \wB^T\wP)\otimes\calD\Big]
(I_{\wN} \otimes Q)^T,
\]
and, as a consequence, the eigenvalues of \(\calT\) are those of the matrix in square brackets in the display. This matrix consists of \(\wN^2\)   blocks, where each block is diagonal of size \(d\times d\). After reordering, the matrix in square brackets becomes a direct sum of the \(d\)  matrices in \eqref{eq:Hmatrix}.
\end{proof}

We now describe how to find, for a given \(\wP\succ 0\), the decay rate \(\lambda\) in \eqref{eq:contractP}. The hypotheses on \(f\) guarantee that, in \eqref{eq:Hmatrix}, \(H_i(y_1,y_2)\in [m,L]\). After defining the matrix-valued function of the real variable \(H\in[m,L]\) given by
\begin{equation}\label{eq:Z}
\wZ(H) = -\wP\big(\wA+H\wB \wC\big)-\big(\wA+H\wB \wC\big)^T\wP,
\end{equation}
we see from Lemma~\ref{lemma:con2} that, if, for each \(H\in[m,L]\),  \(\lambda \wP -\wZ(H)\preceq 0\), then \(\calT\preceq 0\). We factorize \(\wP = \wL\wL^T\) with \(\wL\) invertible; for instance \(\wL\) may be chosen to be lower triangular with positive diagonal entries ---Choleski's factorization---, but other possibilities of course exist. The condition \(\lambda \wP -\wZ(H)\preceq 0\) is  equivalent to the condition
\(\lambda I_d\preceq \wL^{-1}\wZ(H)\wL^{-T}\). Therefore we will have \(\calT\preceq 0\) if, as \(H\) varies in \([m,L]\), the eigenvalues of \(\wL^{-1}\wZ(H)\wL^{-T}\) are positive and bounded away from zero. When that is the case, \(\lambda\) may be chosen to be the infimum of those eigenvalues. We also note that the eigenvalues of
\(\wL^{-1}\wZ(H)\wL^{-T}\) are the eigenvalues of the generalized eigenvalue problem \(\wZ(H)x = \Lambda \wP x\).
To sum up:
\begin{proposition}\label{prop:LZ}
Given the
symmetric, positive definite \(\wP\), define \(\wZ(H)\) by \eqref{eq:Z}. Assume that,
 as \(H\) varies in \([m,L]\), the eigenvalues \(\Lambda\) of the generalized eigenvalue problem
 \(\wZ(H)x = \Lambda \wP x\)  are positive and bounded away from zero and let \(\lambda>0\) be the infimum of those eigenvalues. Then
 the contractivity bound \eqref{eq:contractP} with \(P=\wP\otimes I_d\) holds almost surely.
Alternatively, \(\lambda\) may be defined as the infimum of the eigenvalues of the matrices    \(\wL^{-1}\wZ(H)\wL^{-T}\), where \(\wL\) is any matrix with \(\wP=\wL\wL^T\).
\end{proposition}

The following two examples show this framework  applied to the case of equations \eqref{eq:overdamped} and \eqref{eq:underdamped}.

\begin{example} In the case of the overdamped Langevin equation \eqref{eq:overdamped} if we make the  choice \(\wP = 1\), a simple calculations gives  \(\wZ_i = 2cH_i\). We hence see that in this case \(\lambda = 2cm\), a well-known result.
\end{example}

\begin{example} The paper \cite{CCB18} studies the underdamped Langevin equation \eqref{eq:underdamped} and fixes  \(\gamma = 2\). This does not entail any loss of generality as the value of \(\gamma>0\) may be chosen arbitrarily by rescaling the variable \(t\).\footnote{Other authors, see e.g.\ \cite{DD20}, use different scalings. When we quote estimates from papers that use alternative scalings, we have translated them to the scale in \cite{CCB18} in order to have meaningful comparisons.} Furthermore,
\cite{CCB18} sets \(c = 1/L\) and
\begin{equation}\label{eq:Pforunderdamped}
\wP = \left[\begin{matrix}1&1\\1&2\end{matrix}\right], \qquad \wL = \left[\begin{matrix}1&0\\1&1\end{matrix}\right].
\end{equation}
For these choices, we find
\[
\wL^{-1}\wZ(H)\wL^{-T} = \left[\begin{matrix}2 & H/L-2\\H/L-2&2\end{matrix}\right];
\]
the eigenvalues of this matrix
 are  \( H/L\) and \(4-H/L\) and, since \(H\in[m,L]\), they are \(\geq m/L= 1/\kappa\) (\(\kappa\) denotes the \emph{condition number}). In this case \(\lambda = 1/\kappa\) and \eqref{eq:contractP} becomes
 \begin{eqnarray*}
&&\|x_2(t)-x_1(t)\|^2+ \|x_2(t)+v_2(t)-x_1(t)-v_1(t)\|^2\leq\\ &&\quad\quad\quad\quad\exp(-t/\kappa)\Big(
\|x_2(0)-x_1(0)\|^2+ \|x_2(0)+v_2(0)-x_1(0)-v_1(0)\|^2\Big);
\end{eqnarray*}
which is the contraction estimate used in \cite{CCB18}.

We note that the use of the inner product associated with \(P\) for \((v,x)\) is equivalent to  working  with the variables \((x+v,v)\) and  the standard Euclidean inner product. This \(P\)-inner product often appears in the construction of Lyapunov functions for damped oscillators \cite{SSSt99,BoSS17}
\end{example}

\begin{example}
\label{ex:tuesday}In the setting of the preceding example, we keep \(\gamma =2\) and \(\wP\) as in \eqref{eq:Pforunderdamped}, but do not assume \(c=1/L\). The eigenvalues of the \(2\times 2\) matrix $ \wL^{-1}\wZ(H)\wL^{-T}$ are found to be \(\Lambda^+(H)= cH\) and \(\Lambda^-(H)= 4-cH\); for future reference, we note that they depend on \(H\) and \(c\) through the combination \(cH\) (as it was to be expected from \eqref{eq:underdamped}, where \(\nabla f(x)\) is multiplied by \(c\)). We distinguish four cases:
\begin{enumerate}
\item \(c<4/(L+m)\). As \(H\) varies in \([m,L]\), we have \(\min(\Lambda^+(H)) = cm\) and \modk{\(\min(\Lambda^-(H))=4-cL>cm\)}. Therefore
in this case the \(\lambda=cm\) and an increase in \(c\) results in an increase in \(\lambda\). In particular, for
\(1/L < c < 4/(L+m)\) the contraction rate improves on the value \(1/\kappa\) corresponding to the
choice \(c=1/L\) in \cite{CCB18} discussed in the preceding example.
\item \(c=4/(L+m)\). In this case \(\min(\Lambda^+(H))= cm\) and  \(\min(\Lambda^-(H)) = 4-cL\) have the common value \(4/(\kappa+1)\).
\item  \(c\in [4/(L+m), 4/L)\). Now \(\min(\Lambda^+(H)) = cm\) is larger than \( \min(\Lambda^-(H))=4-cL\) and therefore \(\lambda =4-cL\), which decreases as \(c\) increases.
    \item \(c\geq 4/L\). In this case \(\min(\Lambda^-(H))\leq 0\) and there is no contractivity.
\end{enumerate}
Therefore,  with \(\gamma =2\) and
\(\wP\) in \eqref{eq:Pforunderdamped}, the choice  \(c=4/(L+m)\) yields the best mixing: \(\lambda = 4/(\kappa+1)\). We prove in Section~\ref{secc:decay} that the mixing cannot be improved by using alternative choices of \(\wP\).

More sophisticated choices of \(\wP\) are considered in \cite{DD20}.\footnote{The matrix \(\wP\) is not used in that reference, which only works with a
 non-triangular  \(\wL\) such that \(\wP=\wL^T\wL\). In turn, \(\wL\) is defined indirectly by choosing the columns of \(\wL^{-T}\) to be eigenvectors of a suitable known matrix that depends on a real parameter. The parameter is  tuned to enhance the rate of contraction.}
While those choices allow, for some values of \(c\), a degree of improvement on the value of \(\lambda\) we have obtained by using \eqref{eq:Pforunderdamped} in Proposition~\ref{prop:LZ}, they do not yield values of \(\lambda\) above \(4/(\kappa+1)\) (which is of course in agreement with the analysis in Section~\ref{secc:decay} below). In addition the study in \cite{DD20} assumes that the variable \(v\) is started at stationarity and only monitors the mixing in the variable \(x\).
\end{example}

{\color{black} A useful reference on contractivity is \cite{M20new}.

\begin{rem}
In the examples above it was assumed that \(\wP\) was known at the outset. Due to the small dimension of this matrix in applications, it is not difficult to \emph{find} favourable choices of \(\wP\). This is illustrated in Section~\ref{secc:decay} (see also \cite{DD20}).
\end{rem}
}


\section{Discretizations}
\label{sec:integrators}
Having established  properties for solutions of SDEs of the type \eqref{eq:cont}, we now turn our attention to the properties of their numerical discretizations. We derive a result analogous to Proposition \ref{prop:LZ} to establish the contractivity of the numerical solutions for integrators that use
only one gradient evaluation per time step. Such integrators are particularly attractive in problems of high dimensionality.

\subsection{Discrete state-space form}
\label{sec:statedisc}
To discretize \eqref{eq:cont} on the grid points \(t_n = nh\), \(h>0\), \(n = 0,1,2, \dots\), we use schemes of the form:
\begin{subequations} \label{eq:disc}
\begin{eqnarray}
 \label{eq:sys1ter}  \xi_{n+1} &=&  A_h\xi_n+ B_hu_n+ \sigma^\xi_h \Omega_n, \\
 \label{eq:sys2ter} y_n &=&  C_h\xi_n+ \sigma^y_h \Omega_n, \\
 \label{eq:sys3ter} u_n &=& \nabla f(y_n),
\end{eqnarray}
\end{subequations}
Here, at each step,  \(y_n\in\R^d\) is the feedback output at which the gradient \(\nabla f\) will be evaluated and \(\Omega_n\) represents a random vector in \(\R^{\bar M}\) suitably derived from the restriction to \([t_n,t_{n+1}]\) of  the Brownian motion \(W\) in \eqref{eq:cont}. The real matrices \(A_h\), \(B_h\), \(C_h\),  \(\sigma^\xi_h\) and \(\sigma^y_h\) are constant, with sizes \(N\times N\), \(N\times d\), \(d\times N\), \(N\times \bar M\) and \(d\times \bar M\) respectively. {\color{black} As the examples that follow will illustrate, consistency requires that $h^{-1}(A_h-I)$ be an approximation to $A$ in \eqref{eq:cont}, while $h^{-1}B_h$ and $C_h$ approximate $B$ and $C$. Note also the noise in \eqref{eq:sys2ter}, which has no countepart in \eqref{eq:sys2bis}. }

\begin{example}The Euler-Maruyama scheme for the SDE \eqref{eq:overdamped}
\[x_{n+1} = x_n-hc\nabla f(x_n)+\sqrt{2c}\, (W(t_{n+1})-W(t_n))
\]
is of the form \eqref{eq:disc} with \(N = d\), \(\bar M =d\), \(\xi=x\), \(y=x\), \(\Omega_n =W(t_{n+1})-W(t_n)\), \(A_h = I_d \), \(B_h = -hcI_d\), \(C_h=I_d\),
\(\sigma^\xi_h = \sqrt{2c}I_d\), \(\sigma^y_h = 0_{d\times d}\).
\end{example}

\begin{example} To shorten the notation, we introduce the functions:
\[
\calE(t) = \exp(-\gamma t),\qquad \calF(t) = \int_0^t \calE(s)\,ds =\frac{1-\exp(-\gamma t)}{\gamma},
\]
and
\[
\calG(t) =\int_0^t \calF(s)\,ds = \frac{\gamma t +\exp(-\gamma t)-1}{\gamma^2}.
\]
For the integration of \eqref{eq:underdamped} Cheng et al. \cite{CCB18} use the  scheme:
\begin{subequations}\label{eq:cheng}
\begin{eqnarray}\label{eq:chengv}
v_{n+1} & =& \calE(h) v_n - \calF(h) c\nabla f(x_n)
+\sqrt{2\gamma c} \int_{t_n}^{t_{n+1}} \calE(t_{n+1}-s)dW(s),\\
\label{eq:chengx}
x_{n+1} & = & x_n +  \calF(h) v_n-\calG(h)c\nabla f(x_n)
 + \sqrt{2\gamma c} \int_{t_n}^{t_{n+1}}\calF(t_{n+1}-s) dW(s).
\end{eqnarray}
\end{subequations}
{\color{black} In this example,} \(N = 2d\), \(\bar M = 2d\), \(\xi = [v^T,x^T]^T\), \(y=x\),
\[
\Omega_ n = \left[ \begin{matrix}\int_{t_n}^{t_{n+1}} \calE(t_{n+1}-s)dW(s) \\ \int_{t_n}^{t_{n+1}}\calF(t_{n+1}-s) dW(s) \end{matrix}\right],
\]
and
\[
A_h = \left[\begin{matrix}\calE(h)I_d & 0_{d\times d}\\  {\color{black} \calF(h)I_d}  & {\color{black}I_d}\end{matrix}\right],\qquad
B_h =  \left[\begin{matrix}- \calF(h) cI_d\\-\calG(h)cI_d \end{matrix}\right],
\]
\[C_h = [0_{d\times d}, I_d]
\qquad \sigma_h^\xi = \sqrt{2c\gamma} I_{2d},\qquad \sigma_h^y = 0_{d\times 2d}.
\]
The recipe for simulating the Gaussian random variables  \(\Omega_n\) may be seen in \cite{CCB18}.

In the absence of noise, this integrator is the well-known Euler exponential integrator \cite{HO10}, based, via the variation of constants formula/Duhamel's principle, on the exact integration of the system \(dv/dt =-\gamma v\), \(dx/dt =v\). In the stochastic scenario the algorithm is first order  in both the weak and strong senses. The paper \cite{FLO21} calls this scheme  the left point method. In what follows we shall refer to it as the Euler exponential (EE) integrator.
\end{example}

\begin{example}Another  instance of an underdamped Langevin integrator of the form \eqref{eq:disc} is the following UBU algorithm:
\begin{subequations}\label{eq:ubu}
\begin{eqnarray}\label{eq:ubuv}
v_{n+1} & =& \calE(h) v_n -h \calE(h/2) c \nabla f(y_n)
 +\sqrt{2\gamma c}\int_{t_n}^{t_{n+1}} \calE(t_{n+1}-s)dW(s),\\
 \label{eq:ubux}
x_{n+1} & = & x_n + \calF(h) v_n-h \calF(h/2) c \nabla f(y_n)
+ \sqrt{2\gamma c} \int_{t_n}^{t_{n+1}}\calF(t_{n+1}-s) dW(s),\\
\label{eq:ubuy}
y_n & = & x_n + \calF( h/2)v_n+
\sqrt{2\gamma c} \int_{t_n}^{t_{n+1/2}}\calF(t_{n+1/2}-s) dW(s).
\end{eqnarray}
\end{subequations}
Here and later \(t_{n+1/2} = t_n+h/2\).
UBU is a splitting integrator \cite{MQ02} that  is second order  in both the weak and strong senses. See Section~\ref{sec:UBU} for details and note that both EE and UBU use stochastic integrals of the form
\(\int \calF dW\) and therefore are not covered by the analysis in \cite{CaoLuWang}.
\end{example}

\begin{rem}\label{rem:moreeval}
The format \eqref{eq:disc} only caters for schemes that use a single evaluation of the gradient \(\nabla f\) per step. By increasing the dimension of \(u\), the format may be easily adapted to integrators that use several gradient evaluations, cf.\ \cite{LRP16,FRMP,SSKZ20a}. However, the technique used below to establish the contractivity of the integrators cannot be immediately extended to schemes with several gradient evaluations; several gradient evaluations would bring in Hessian matrices evaluated at different locations and it would not be possible to diagonalize those Hessians simultaneously, as we did when proving Lemma~\ref{lemma:con2}.
 For the contractivity of algorithms involving several gradient evaluations see e.g. \cite{SSKZ20} and its references.
\end{rem}

\subsection{The evolution of probability distributions in the discrete case}
We will denote by \(\Psi_{h,n}\pi\) the probability distribution for \(\xi_n\) in \eqref{eq:disc} when \(\pi\) is the distribution of \(\xi_0\) {\color{black} (thus \(\Psi_{h,n}\pi\) is an operator on measures)}. After introducing (cf.\ \eqref{eq:twice})
\begin{subequations} \label{eq:twicedisc}
\begin{eqnarray}
 \xi^{(1)}_{n+1} &=&  A_h\xi^{(1)}_n+ B_h\nabla f(C_h\xi^{(1)}_n+\sigma_h^y\Omega_n)+ \sigma_h^\xi \Omega_n,\\
 \qquad \xi^{(2)}_{n+1} &=&  A_h\xi^{(2)}_n+ B_h\nabla f(C_h\xi^{(2)}_n+\sigma_h^y\Omega_n)+ \sigma_h^\xi \Omega_n,
\end{eqnarray}
\end{subequations}
we have the following discrete counterpart of Proposition~\ref{prop:twice}, whose proof will not be given:
\begin{proposition}\label{prop:twicedisc}
Assume that \(P_h\succ 0\) and \(\rho_h\in(0,1)\) exist such that for \eqref{eq:twicedisc}, almost surely,
\begin{equation}\label{eq:contractPdisc}
\|\xi^{(2)}_{n+1} -\xi^{(1)}_{n+1}\|_{P_h}^2 \leq \rho_h \|\xi^{(2)}_n -\xi^{(1)}_n\|_{P_h}^2,\qquad n=0,1,\dots
\end{equation}
Then, for arbitrary distributions, \(\pi_1\) and \(\pi_2\),
\begin{equation}\label{eq:contractwasserdis}
W_P(\Psi_{h,n}\pi_1,\Psi_{h,n}\pi_2) \leq \rho_h^{n/2}W_P(\pi_1,\pi_2),\qquad n=0,1,\dots
\end{equation}
\end{proposition}
\subsection{Checking discrete contractivity}
The proof of the following result is similar to that of Proposition~\ref{prop:contrac} and will be omitted:
\begin{proposition}\label{prop:contracdisc}
Let \(P_h\succ 0\) be an \(N\times N\) symmetric matrix and \(\rho_h\in(0,1)\).
Assume that, for each \(y_1,y_2\in\R^d\) the matrix
\[
\calT_h(\rho_h,P_h,y_1,y_2) = \rho_h P_h-\big( A_h + B_h \bar\calH(y_1,y_2)C_h\big)^T P_h\big( A_h + B_h \modk{\bar\calH}(y_1,y_2)C_h  \big)
\]
is \(\succeq 0\).  Then,  for solutions of \eqref{eq:twicedisc} the contractivity estimate \eqref{eq:contractPdisc} holds almost surely.
\end{proposition}

In a similar way as for the continuous time case we can prove a discrete time counterpart for  Proposition~\ref{prop:LZ}.
\begin{proposition}\label{prop:LZdisc}
Given the
symmetric, positive definite \(\wP_h\), set
 \[
 \wZ_{h}(H) =\big( \wA_h +  H\wB_h\wC_h  \big)^T \wP_h\big( \wA_h +  H\wB_h\wC_h  \big).
 \]
 Assume that,
 as \(H\) varies in \([m,L]\), the supremum \(\rho_h\) of the eigenvalues \(R\) of the generalized eigenvalue problems
 \(\wZ_{h}(H)x = R \wP x\) is \(<1\). Then
 the contractivity bound \eqref{eq:contractPdisc} with \(P_h=\wP_h\otimes I_d\)  holds almost surely.
Alternatively, \(\rho_h\) may be defined as the suppremum of the eigenvalues of the matrices    \(\wL^{-1}_h\wZ_{h}(H)\wL^{-T}_h\), where \(\wL_h\) is any matrix with \(\wP_h=\wL_h\wL^T_h\).
\end{proposition}

The remainder of this section is devoted to the application of the last proposition to the investigation of the contractivity of the integrators \eqref{eq:cheng} or \eqref{eq:ubu} applied to the underdamped Langevin system \eqref{eq:underdamped}
with \(\gamma = 2\) and \(\wP_h\) chosen to coincide with \(\wP\) in \eqref{eq:Pforunderdamped}. We have computed symbolically the eigenvalues \(R = R_h^\pm(H)\) in the proposition in closed form, but the resulting expressions are complicated and will not be reproduced here. (For each fixed \(H\) we attach the \(+\) superscript to the discrete eigenvalue \(R(H)\) closest to \(\Lambda^+(H) =cH\) and the \(-\) superscript to the other.)
Rather than  analysing directly the discrete eigenvalues, we follow an alternative approach based on leveraging the contractivity of the SDE (studied in Example~\ref{ex:tuesday}) and the \emph{consistency of the discretizations}.
The key observation is that, by definition of consistency, for fixed \(H\in[m,L]\) and as \(h\downarrow 0\),   the numerical propagator matrix  \(\wA_h +  H\wB_h\wC_h\)  in Proposition~\ref{prop:LZdisc} differs from the differential equation propagator \(\exp(-h (\wA +  H\wB\wC))\) in Proposition~\ref{prop:LZ}  by an \(\mathcal{O}(h^{p+1})\) amount, where \(p=1\) for the first order EE integrator and \(p=2\) for the second order UBU. As a consequence,
\(- h^{-1} \log(\wA_h +  H\wB_h\wC_h)\) is \(\mathcal{O}(h^{p})\) away from  \(\wA +  H\wB\wC\) (cf.\ \cite[Example 10.1]{sanz2018numerical}), and, for the eigenvalues, we have \(-h^{-1}\log(R_h^\pm(H)) = \Lambda^\pm(H)+\mathcal{O}(h^{p})\). It is convenient for our purposes to work, rather than with
 \(-h^{-1}\log(R_h^\pm(H))\), in terms of the quantities
\begin{equation}\label{eq:friday}
\widetilde\Lambda^\pm_h(H) = 2h^{-1}(1-R_h^\pm(H)^{1/2});
\end{equation}
 for these (since, as \(\zeta \rightarrow 1\), \(-h\log \zeta \sim 2 (1-\zeta^{1/2})\)) we have
\(\widetilde \Lambda_h^\pm(H)= \Lambda^\pm(H)+\mathcal{O}(h)\). An illustration of the convergence of
 \(\widetilde\Lambda_h^\pm(H)\) to \(\Lambda^\pm(H)\) may be seen in Figure~\ref{fig:fig1}.

\begin{rem}Note \modk{that}  \(R_h^\pm(H)^{1/2} = 1- \widetilde \Lambda_h^\pm h/2\)
is an approximation to \(\exp(-\Lambda^\pm(H) h/2)\) and compare with the relation  between the discrete decay factor \(\rho_h^{1/2}\) over one time step in \eqref{eq:contractwasserdis} and the SDE decay factor \(\exp(-\lambda h/2)\) in
\eqref{eq:contractwasser}  over a time interval of length \(h\).\end{rem}

 Because the discrete eigenvalues depend smoothly on \(H\) and this variable ranges in the compact interval \([m,L]\), the convergence \(\widetilde \Lambda_h^\pm(H)\rightarrow \Lambda^\pm(H)\) is uniform in \(H\). Therefore \(2(1-\rho_h^{1/2})/h\), which is the minimum of \(2(1-R_h^\pm(H)^{1/2})/h\), converges to  the minimum of \(\Lambda^\pm(H)\), in the limit where
\(h\downarrow 0\) \emph{with \(c\), \(m\) and \(L\) fixed}. We conclude that, for \(h\) small enough, the discretizations will behave contractively when the SDE does, i.e.\ whenever \(c\leq 4/(L+m)\). However, this conclusion is \emph{per se}  rather weak because, as we have seen in Example~\ref{ex:tuesday}, as \(L\) and \(m\) vary with \(\kappa\rightarrow \infty\), the contraction rate \(\lambda\) behaves like \(\mathcal{O}(\kappa^{-1})\) and it may be feared that the discretizations  be contractive only for values of  \(h\) that, as \(\kappa\) increases unboundedly, tend to \(0\) . If that were the case, the usefulness  of the integrators \eqref{eq:cheng} or \eqref{eq:ubu} could be doubted. The next result proves that those fears are not warranted if \(c\) is chosen appropiately: \emph{the discretizations are contractive with a rate that essentially coincides with the SDE rate}, provided that \(h\) is below a threshold independent of \(L\) and \(m\).

\begin{figure}[t]

\begin{center}
\includegraphics[scale=0.40]{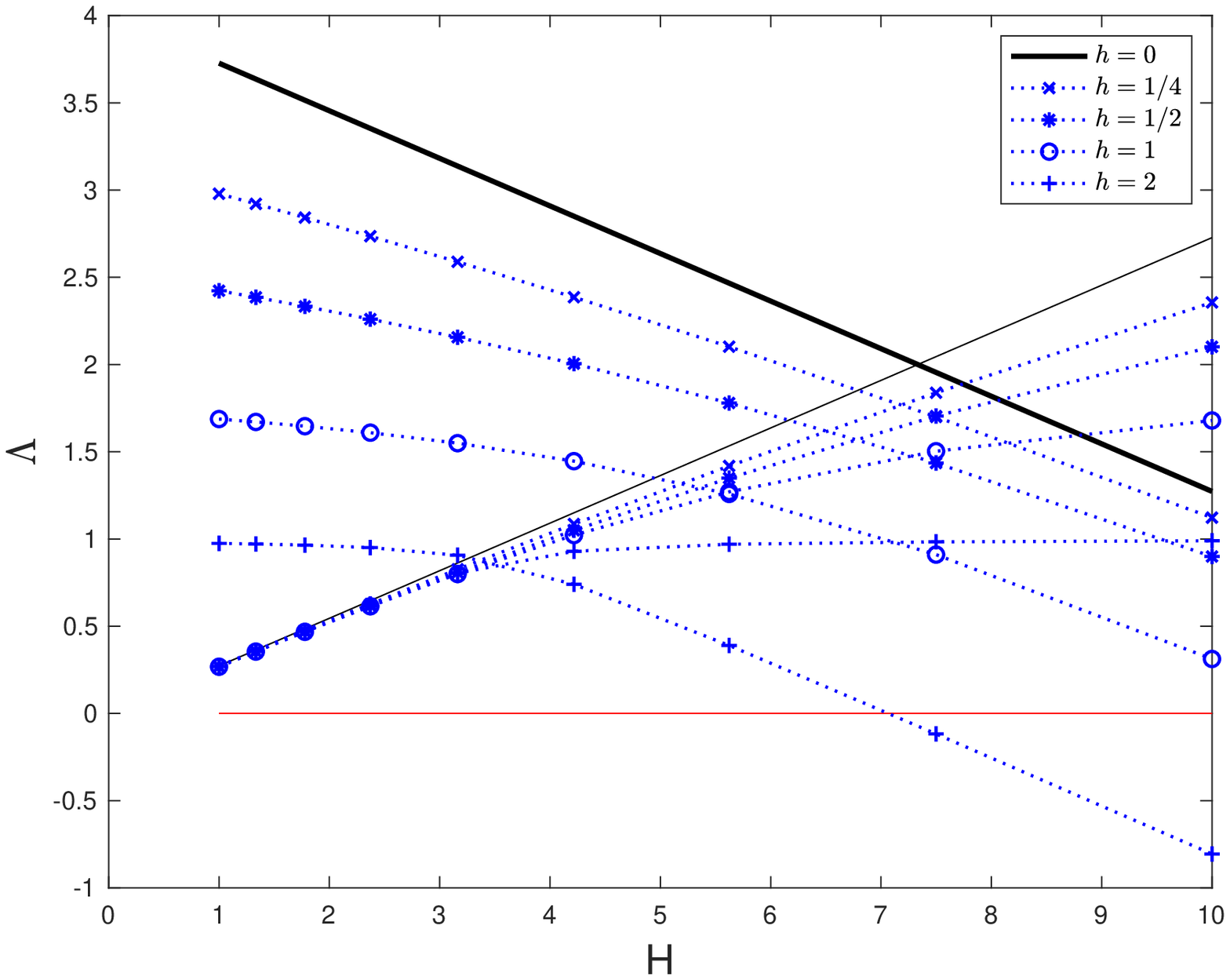}
\includegraphics[scale=0.40]{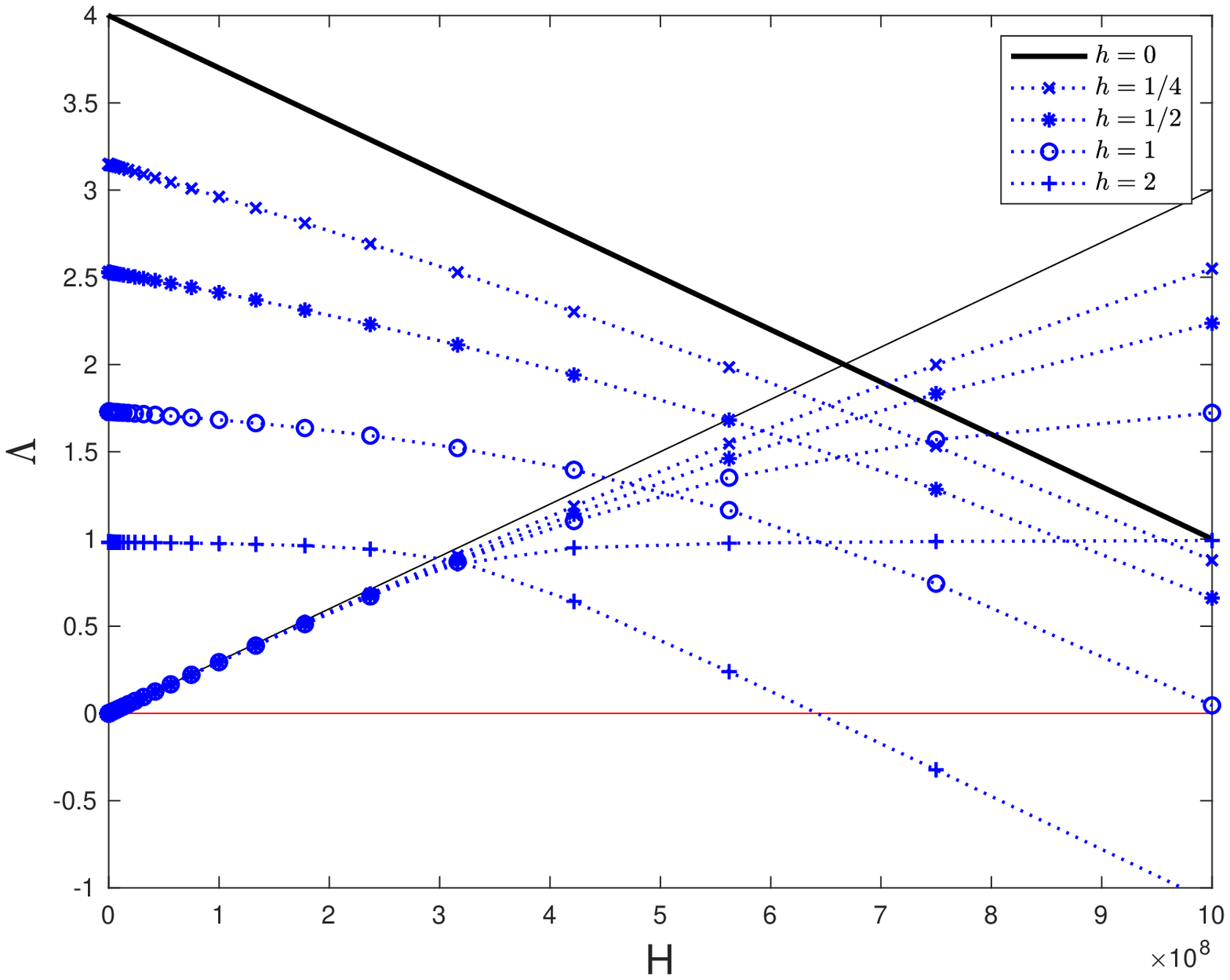}
\end{center}
\caption{ On the left, \(\gamma = 2\), \(L=10\), \(m=1\), \(c = 3/(L+m)\) and \(\wP\) is as in \eqref{eq:Pforunderdamped}; the condition number is very low so as to enhance the clarity of the figure. The solid lines correspond to the SDE eigenvalues \(\Lambda^+(H) = cH\), \(\Lambda^-(H) = 4-cH\).  The value of \(\lambda\) is the minimum eigenvalue and  occurs for \(\Lambda^+\) evaluated at \(m\), so that \(\lambda = 3m/(L+m) =3/11 \). The discontinuous lines represent the UBU discrete counterparts  \(\widetilde\Lambda_h^+(H)\) and \(\widetilde\Lambda_h^-(H)\) for \(h= 2, 1, 1/2, 1/4\); as \(h\) decreases, \(\widetilde\Lambda_h^+(H)\) and \(\widetilde\Lambda_h^-(H)\) converge to the SDE eigenvalues. For \(h=2\) and \(H\) large, \(\widetilde\Lambda_h^-(H) <0\) and the numerical scheme is not contractive. The parameters in the right panel are the same as those on the left, with the exception that \(L=10^9\) leading to an extremely high condition number. Now the minimum of the SDE eigenvalues is \(\lambda \approx 10^{-9} \). As on the left, there is numerical contractivity for \(h=1,1/2,1/4\), but not for \(h=2\); see the leftmost column in Table~\ref{tb:table1}.}
\label{fig:fig1}
\end{figure}

\begin{theorem}\label{theo:thursday}
Consider the SDE \eqref{eq:underdamped} with \(\gamma=2\), \(c = \bar c/(L+m)\), where the constant \(\bar c\in(0,4)\) is independent of \(L\) and \(m\). For
the discretization provided by the integrators \eqref{eq:cheng} or \eqref{eq:ubu}, to any \(\bar r< \bar c/2\) there corresponds a value \(h_0=h_0(\bar r)\) such that, for \(h\leq h_0\),  the discrete contraction estimate \eqref{eq:contractPdisc} holds with \(P_h = \wP\otimes I_d\) (\(\wP\) is the matrix in \eqref{eq:Pforunderdamped}) and \(\rho_h = 1-\bar r h/(\kappa+1)\).
\end{theorem}

\begin{proof}We begin by recalling, from Example~\ref{ex:tuesday}, that the SDE eigenvalues are \(\Lambda^+ = cH\) and \(\Lambda^- = 4-cH\).
In addition, and for the reasons we pointed out in the continuous case, the discrete eigenvalues \(R_h^\pm(H)\), for fixed \(h\), are functions of the combination \(\widetilde H = cH\). In other words, for fixed \(\widetilde H\), \(R_h^\pm\) depend \emph{only} on \(h\) (i.e.\ they are independent of \(L\) and \(m\)). In this way, when thinking in terms of the variable \(\widetilde H\), changing \(L\) and \(m\) only impacts the convergence \(\widetilde\Lambda_h^\pm\rightarrow \Lambda^\pm\) by changing the  interval \(\modk{[cm, cL]}=[\bar c/(\kappa+1), \bar c\kappa/(\kappa+1)]\subset [0,\bar c]\) of values of \(\widetilde H\) that have to be considered. Therefore, how much  \(h\) has to be reduced to get \(|\widetilde \Lambda^\pm_h(H)-\Lambda^\pm(H)|\) below a target error tolerance is independent of \(H\in[m,L]\), \(m>0\) and \(L\geq m\).

 Now the theorem is clearly true when \(\kappa\)  ranges in a bounded interval and we may suppose in what follows that \(\kappa\) is so large that \(\bar c/(\kappa+1) \leq 2-\bar c/2\). We consider the two eigenvalues \(-\) and  \(+\) successively.

For \(\widetilde \Lambda_h^-\) we note that, as \(H\) varies in \([m,L]\), \[\Lambda^-(H)\geq \Lambda^-(L)= 4-cL \modk{=} 4- \frac{\bar c L}{L+m}= \frac{(4-\bar c) L+4m}{L+m}= \frac{(4-\bar c)\kappa +4}{\kappa+1}\geq 4-\bar c.\]
As a consequence, for
 \(h\) small enough, \(\widetilde \Lambda_h^-(H) >(1/2)(4-\bar c)\). In view of the restriction on \(\kappa\),
\(\widetilde \Lambda_h^-(H) >\bar c/(\kappa+1)\).

The discussion of the behaviour of \(\widetilde \Lambda_h^+\) is more delicate. Here we need to take into account that, if we set \(c=0\) in \eqref{eq:overdamped} so as to switch off the force \(\nabla f(x)\) and the noise, then both integrators under consideration are exact. This implies that, at \(\widetilde H=0\) and for arbitrary \(h\),  \(\widetilde \Lambda_h^+\) coincides exactly with the continuous eigenvalue \(\Lambda^+ = 0\). This in turn entails that the  error
\(\widetilde \Lambda_h^+(H)-\Lambda^+(H)\) vanishes at \(\widetilde H=0\) for each \(h\) and
must then have an expression of the form \(h\widetilde H G(h,\widetilde H)\), where \(G\) is a smooth function (this is born out in Figure~\ref{fig:fig1}, where the difference \(\modk{\widetilde \Lambda_h^+(H)}-\Lambda^+(H)\) decreases as \(\widetilde H\) decreases with fixed \(h\)). Since \(\Lambda^+(H) = \widetilde H\), for the \emph{relative} error we may write
\(\big(\widetilde \Lambda_h^+(H)-\Lambda^+(H)\big)/ \Lambda^+(H)=hG(h,\widetilde H)\). By taking \(h\) sufficiently small we may guarantee that \(\widetilde \Lambda_h^+(H)\geq (2\bar r/\bar c) \Lambda^+(H)\)
and, since \(\Lambda^+(H)\geq \Lambda^+(m) =cm = \bar c /(\kappa+1)\), we have \(\widetilde \Lambda_h^+(H) >\bar c/(\kappa+1)\) and the proof is complete.
\end{proof}

\begin{rem}\label{rem:friday}The same proof shows that if \(c=1/L\) (as in \cite{CCB18}) a similar results holds with a rate \(\rho_h = 1-\bar r  h/\kappa\), where \(\bar r<1/2\) may be chosen arbitrarily.
\end{rem}

\begin{rem}\label{rem:sunday} The choice \(c=4/(L+m)\) guarantees contractivity in the SDE, but has to be excluded from Theorem~\ref{theo:thursday}. For this value, the proof breaks down because, as the condition number increases, \(\Lambda^-(L) = 4m/(L+m)\) is not bounded away from zero. By using the expressions of the eigenvalues \(R^\pm_h(H)\) at \(H=L\), it may be seen that contractivity  requires \(h=\mathcal{O}(\kappa^{-1})\) for the EE integrator and \(h=\mathcal{O}(\kappa^{-1/2})\) for the second order UBU.
\end{rem}

\begin{rem}Only two properties of the integrators EE and UBU have been used in the proof: (i) they are consistent, (ii) they are exact if the force and noise are switched off.  The second of these was required to prove that, for each \(h\), \(\widetilde \Lambda_h^+(0)=0\), or equivalently, \(R_h^+(0) = 1\) (see \eqref{eq:friday}), which means that
that \(\wA_h\) has \( 1\) as an eigenvalue. In fact, for all reasonable discretizations, it is true that \(\wA_h\) has the eigenvalue \(1\). This happens because with \(v_0 = 0\) and \(c=0\) (no velocity, no force) any reasonable discretization will yield \(v_1=v_0\), \(x_1=x_0\) (the particle stands still). Therefore \emph{the theorem  is true for  all integrators of interest}.
\end{rem}

\begin{table}
\label{table}
\begin{center}
\begin{tabular}{c|cc|cc|cc}
$h$ & \multicolumn{2}{c|} {$c=1/L$}
&\multicolumn{2}{c|} {$c= 2/(L+m)$}
&\multicolumn{2}{c} {$c= 3/(L+m)$}\\
\hline
& EE & UBU & EE  & UBU &EE & UBU\\
\hline
2& *** & 5.000(-10)& *** & ***&***&***\\
1& 5.000(-10)& 5.000(-10) & *** & 1.000(-9)& ***&1.500(-9)\\
1/2 & 5.000(-10) &5.000(-10) & 1.000(-9) & 1.000(-9)& ***&1.500(-9)\\
1/4 & 5.000(-10) &5.000(-10) & 1.000(-9) & 1.000(-9)&1.500(-9)& 1.500(-9)
\end{tabular}

\end{center}
\caption{\small Contractivity of the integrators EE and UBU for the underdamped Langevin equations with  \(\gamma=2\), \(P_h=\wP\otimes I_d\) (\(\wP\) as in \eqref{eq:Pforunderdamped}). The table gives the value of \((1-\rho_h^{1/2})/h\) where \(\rho_h\) is as in  \eqref{eq:contractPdisc}. The symbol *** means that for that combination of \(h\) and \(c\) the integrator is not contractive. The table is for the large condition number \(\kappa=10^9\). In the corresponding tables for
\(\kappa = 10^3\) or \(\kappa=10^6\) (not reproduced in this paper),  the symbol *** appears  at exactly the same locations as in the case \(\kappa=10^9\) reported above,  but the values of \((1-\rho_h^{1/2})/h\) are multiplied by \(10^6\) or \(10^3\) respectively, showing a \(1/\kappa\) behaviour.}
\label{tb:table1}
\end{table}

\begin{example} The proof of the theorem sheds no light on the size of the threshold \(h_0\). With a view  to
ascertaining the range of values of \(h\) for which the integrators \eqref{eq:cheng} and \eqref{eq:ubu} behave contractively, we have  computed numerically  the eigenvalues in Proposition~\ref{prop:LZdisc} and taken the suppremum over \(H\) with fixed \(h\).
Table~\ref{tb:table1} provides information on the quotient
\((1-\rho_h^{1/2})/h\) that approximates the decay constant \(\lambda/2\)
in the estimate \eqref{eq:contractwasser}. For the parameter choice \(c=1/L\) used in \cite{CCB18}, the table shows that, for \(h\) sufficiently small (say \(h\leq 1\)), there is numerical contractivity and the quotient coincides (within the number of digits reported) with the SDE rate \(\lambda/2 = 1/(2\kappa)\) found in Section~\ref{sec:checkincontractivity}.
The table also gives results for the choices \(c=2/(L+m)\) and \(c=3/(L+m)\), where again the values of \((1-\rho_h^{1/2})/h\) reported are in agreement with the SDE rates \(\lambda/ 2 = 1/(\kappa+1)\) and \(\lambda/ 2 = (3/2)/(\kappa+1)\) found in Example~\ref{ex:tuesday}.

 For all choices of \(c\), the UBU integrator \eqref{eq:ubu} operates contractively for values of \(h\) that are larger than those required by the EE integrator \eqref{eq:cheng}. In addition, as \(c\) increases with fixed \(h\), EE looses contractivity before UBU. This is consistent with Remark~\ref{rem:sunday}.
 \end{example}


\section{A general theorem}
\label{sec:main}
In this section we consider integrators for \eqref{eq:cont} of the form \(\xi_{n+1}= \psi_h(\xi_n,t_n)\), \(t_n=nh\), \(n=0,1,\dots\), \(h>0\), where, following the terminology in \cite{MT04}, \(\psi_h(\xi,t)\) represents the \emph{one-step approximation}; \(\psi_h(\xi,t)\) uses the restriction of the Brownian motion \(W\) in \eqref{eq:cont} to the interval \([t,t+h]\), but this fact is not reflected in the notation. Integrators of the form \eqref{eq:disc} provide a particular class of integrators of this form (cf.\ Remark~\ref{rem:moreeval}). If \(P_h\) is an \(N\times N\) matrix \(\succ 0\) and \(\pi\) is an arbitrary probability distribution for the initial condition \(\xi_0\), we wish to study the distance \(W_{P_h}(\pi^\star, \Psi_{h,n} \pi)\) between the invariant distribution \(\pi^\star\) and the distribution \(\Psi_{h,n}\pi\) of \(\xi_n\).

In the analysis, for random vectors \(X\in\R^N\), we use the Hilbert-space norm
\(\| X\|_{L^2,P_h} = \E(\|X\|_{P_h}^2)^{1/2}\). The symbol \(\langle\cdot,\cdot\rangle_{L^2,P_h}\) will be used for the corresponding inner product.
We denote by \(\phi_h(\xi,t)\) the exact counterpart of \(\psi_h(\xi,t)\), so that if
\(\xi(t)\) is a solution of \eqref{eq:cont} then
\(\xi(t_{n+1}) = \phi_h(\xi(t_n), t_n)\). At each time-level \(n\), \(n=0,1,2, \dots\), we introduce a random variable \(\widehat{\xi}_n\sim \pi^\star\)  such that \(W_P(\pi^\star,\Psi_{h,n} \pi) = \|\widehat{\xi}_n-\xi_n \|_{L^2,P_h}\). For the difference \(\phi_h(\widehat{\xi}_n,t_n)-\psi_h(\widehat{\xi}_n,t_n)\) (that may be seen as a local error), we consider the following assumption:

\begin{assumptions} There is a decomposition
\[\phi_h(\widehat{\xi_n},t_n)-\psi_h(\widehat{\xi}_n,t_n)
=\alpha_h(\widehat{\xi}_n,t_n)+\beta_h(\widehat{\xi}_n,t_n),\]
  and positive constants \(p\), \(h_0\), \(C_0\), \(C_1\), \(C_2\)  such that for  \(n\geq 0\) and \(h\leq h_0\):
 \begin{equation}\label{eq:LE1}
\left| \Big\langle \psi_h(\widehat{\xi}_n,t_n)- \psi_h(\xi_n,t_n),\alpha_h(\widehat{\xi}_n,t_n)\Big\rangle_{L^2,P_h} \right|
\leq C_0 h\:  \|\widehat{\xi}_n-\xi_n \|_{L^2,P_h}\: \|\alpha_h(\widehat{\xi}_n,t_n)\|_{L^2,P_h}
\end{equation}
and
\begin{equation}\label{eq:LE2}
\| \alpha_h(\widehat{\xi}_n,t_n)\|_{L^2,P_h}\leq C_1 h^{p+1/2},\qquad
\| \beta_h(\widehat{\xi}_n,t_n)\|_{L^2,P_h}\leq C_2  h^{p+1}.
\end{equation}
\label{as4}
\end{assumptions}
We can now state our general theorem that gives a bound for the Wasserstein distance between the invariant measure $\pi^{\star}$ and the distribution of the $n+1$-th iteration of the numerical scheme.
\begin{theorem}\label{theo:main}
Assume that the integrator satisfies Assumption \ref{as4} and in addition, there are constants \(h_0>0\),
\(r>0\)  such that for \(h\leq h_0\) the contractivity estimate
\eqref{eq:contractPdisc} holds with  \(\rho_h\leq (1-rh)^2\). Then, for any initial distribution \(\pi\), stepsize \(h\leq h_0\), and \(n=0,1,\dots\),
\begin{equation}\label{eq:mainestimate}
{W_{P_h}(\pi^\star, \Psi_{h,n} \pi)\leq (1-hR_h)^nW_{P_h}(\pi^\star,\pi)+\left(\frac{\sqrt{2}C_1}{\sqrt{R_h}}+\frac{C_2}{R_h}\right) h^p,}
\end{equation}
with
\[
R_h = \frac{1}{h} \Big(1-\sqrt{(1-rh)^2+C_0h^2}\Big) = r+o(1),\quad {\rm as}\quad h\downarrow 0.
\]
\end{theorem}
\begin{figure}
\centering{
\includegraphics[scale=1]{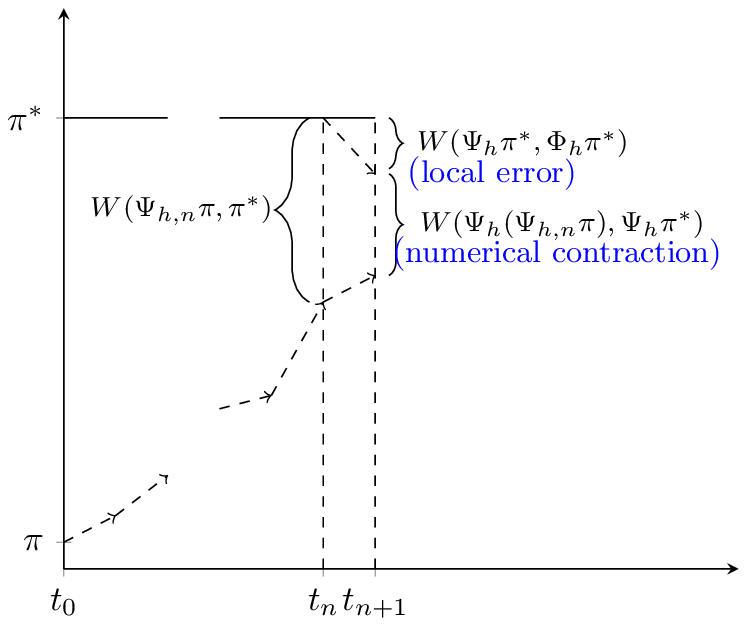}
\includegraphics[scale=1]{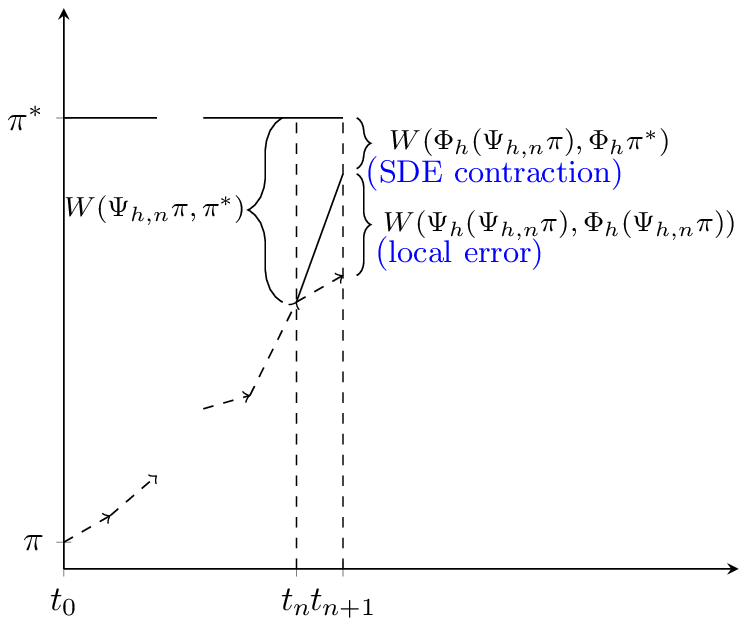}}
\caption{
Different approaches for obtaining a bound for $W_{2}(\Psi_{h,n+1}\pi,\pi^{*})$. The vertical axes represent probability distributions and the horizontal axes correspond to time. Solid lines indicate evolution with the SDE and broken lines evolution with the integrator. On the left, the technique in \cite{D17a,DK19} and this paper, where the contractivity of the integrator is used to propagate forward the
distance between the target distribution \(\pi^\star\) and distribution \(\Psi_{h,n}\pi\) of the numerical solution at time \(t_n\).
On the right, the technique in \cite{CCB18,DD20} that leverages the contractivity of the SDE. On the left, the \lq\lq local error\rq\rq\ is based at the target; on the right is based at the numerical approximation. }
\label{fig:schem}
\end{figure}

Before going into the proof we make some observations:
\begin{itemize}
\item The theorem is in the spirit of classic \lq\lq consistency and stability imply convergence\rq\rq\ numerical analysis results. The main idea of the proof is schematically illustrated in the left panel of Figure~\ref{fig:schem}. The error of interest at time level $n+1$ can be decomposed into two terms. The first is the distance between $\Psi_{h}(\Psi_{h,n}\pi)$ and $\Psi_{h}\pi^{\star}$ and can be bounded in terms of the error at time level $n$ by using the contractivity \emph{of the numerical integrator}. The second is the distance between $\Psi_{h}\pi^{\star}$ and $\Phi_{h}\pi^{\star}=\pi^\star$ and can be bounded by estimating the strong local error by means of Assumption \ref{as4}.
\item The local error needs to be bounded in the \emph{strong sense}. This should be compared with studies about weak convergence of the numerical distribution \cite{AGZ14,LS16} where the  estimates depend on the \emph{weak order} of the integrator.
\item The \(\beta= \mathcal{O}(h^{p+1})\) part  of the  local error results in an \(\mathcal{O}(h^p)\) contribution to the bound on \(W_{P_h}(\pi^\star, \Psi_{h,n} \pi)\). One power of \(h\) is lost from going from local to global as in the classical analysis of (deterministic) numerical integrators.
\item The \(\alpha= \mathcal{O}(h^{p+1/2})\) part of the local error is asked to satisfy the requirement \eqref{eq:LE1} and only looses a factor \(h^{1/2}\) when going from local to global. This is reminiscent of the situation for the strong convergence of numerical solutions of SDEs (see e.g.\ \cite[Theorem 1.1]{MT04}), where, for instance, the Euler-Maruyama scheme with \(\mathcal{O}(h^{3/2})\) strong local errors
    yields \(\mathcal{O}(h)\) strong global errors (assuming additivity of the noise). Typically, the \(\alpha\) part of the local error will consist of Ito integrals that, conditional on events occurring up to the beginning of the time step, have zero expectation.
\end{itemize}

\begin{proof}
We may write
\begin{eqnarray*} \phi_h(\widehat{\xi}_n,t_n) -\xi_{n+1} &= &\big(\psi_h(\widehat{\xi}_n,t_n)- \psi_h(\xi_n,t_n)\big) + \big(\phi_h(\widehat{\xi}_n,t_n)-\psi_h(\widehat{\xi}_n,t_n)\big)\\
&=& \big(\psi_h(\widehat{\xi}_n,t_n)- \psi_h(\xi_n,t_n)+\alpha_h(\widehat{\xi}_n,t_n)\big)+\beta_h(\widehat{\xi}_n,t_n),
\end{eqnarray*}
and therefore, by the triangle inequality and \eqref{eq:LE1}, for \(h\leq h_0\),
\begin{align*}
&\|\phi_h(\widehat{\xi}_n,t_n) -\xi_{n+1}\|_{L^2,P_h} \\
 &\qquad\leq\| \psi_h(\widehat{\xi}_n,t_n)- \psi_h(\xi_n,t_n)+\alpha_h(\widehat{\xi}_n,t_n)\|_{L^2,P_h}+ \|\beta_h(\widehat{\xi}_n,t_n)\|_{L^2,P_h}\\
&\qquad \leq \big(\| \psi_h(\widehat{\xi}_n,t_n)- \psi_h(\xi_n,t_n)\|_{L^2,P_h}^2\\%
&\qquad\qquad\qquad+2C_0 h\|\widehat{\xi}_n-\xi_n \|_{L^2,P_h}
 \| \alpha_h(\widehat{\xi}_n,t_n)\|_{L^2,P_h}+\|\alpha_h(\widehat{\xi}_n,t_n)\|^2_{L^2,P_h}\big)^{1/2}\\
 &\qquad\quad+ \|\beta_h(\widehat{\xi}_n,t_n)\|_{L^2,P_h}.
\end{align*}
We next apply the contractivity hypothesis, the elementary bound \(2ab\leq a^2+b^2\), and \eqref{eq:LE2}:
\begin{align*}
&\|\phi_h(\widehat{\xi}_n,t_n) -\xi_{n+1}\|_{L^2,P_h}\\
&
\quad \leq \Big( (1-rh)^2 \|\widehat{\xi}_n-\xi_n \|_{L^2,P_h}^2+C_0^2h^2\|\widehat{\xi}_n-\xi_n \|_{L^2,P_h}^2
+2 \|\alpha_h(\widehat{\xi}_n,t_n)\|^2_{L^2,P_h}\Big)^{1/2}\\
&\qquad\qquad+ \|\beta_h(\widehat{\xi}_n,t_n)\|_{L^2,P_h}\\
& \quad\leq \Big( \big((1-rh)^2+C_0h^2\big) \|\widehat{\xi}_n-\xi_n \|_{L^2,P_h}^2+2C_1^2h^{2p+1} \Big)^{1/2}+C_2h^{p+1}.
\end{align*}
Therefore, in view of the choice of \(\widehat{\xi}_n\), and taking into account that \(\phi_h(\widehat{\xi}_n,t_n)\sim \pi^*\),
\begin{align*}
&W_{P_h}(\pi^\star, \Psi_{h,n+1} \pi)\\
&\qquad\qquad\leq\Big( \big((1-rh)^2+C_0h^2\big)W_{P_h}(\pi^\star, \Psi_{h,n} \pi)^2
+ 2C_1^2 h^{2p+1}\Big)^{1/2}+ C_2 h^{p+1}.
\end{align*}
The conclusion follows after applying the lemma in Section~\ref{sec:lemma}.
\end{proof}

The paper \cite{DK19} analyzes the Euler-Maruyama discretization of \eqref{eq:overdamped}. Under two different smoothness assumptions on \(f\), it derives two different estimates similar to \eqref{eq:mainestimate}, one with \(p=1/2\) and the other with \(p=1\). {\color{black} The application of Theorem~\ref{theo:main}  to those two cases retrieves the estimates in \cite{DK19}; in addition the proof of our theorem as applied to those two particular cases coincides with the proofs provided in  that paper. (In fact we derived Theorem~\ref{theo:main} as a generalization  of the material in
\cite{DK19} to a more general scenario.)

By considering the case where the target distribution is a product of uncorrelated univariate Gaussians, one sees that the estimates of the mixing time in \cite{DK19} are optimal in their dependence on \(\epsilon\) and \(d\). \cite{DMM18} have shown that those estimates are not optimal in their dependence on \(\kappa\). This implies that the result in Theorem~\ref{theo:main} is not necessarily the best that may be achieved in each application.
}

\section{Application to underdamped Langevin dynamics}
\label{sec:applunderdamped}

We now apply Theorem~\ref{theo:main} to integrators for \eqref{eq:underdamped}.

\subsection{The EE integrator}

We begin with the EE integrator \eqref{eq:cheng}. For its local error we have the following theorem, proved in Section~\ref{sec:LEcheng}. Recall that setting \(\gamma=2\) does not restrict the generality, as such a choice is equivalent to choosing the units of time. Once this value of \(\gamma\) has been fixed, the choice of \(P_h\) in the theorem that follows is the one that allows the best contraction rate for the SDE \eqref{eq:underdamped}.

\begin{theorem}
 Set  \(\gamma = 2\) and \(P_h=\wP\otimes I_d\) with \(\wP\) as in \eqref{eq:Pforunderdamped}.
Then, for \(h\leq 1\), the discretization \eqref{eq:cheng} satisfies Assumption~\ref{as4} with \(p=1\), \(C_0=C_1=0\)
and \(
C_2 = K c^{3/2}L d^{1/2}
\), where \(K\) is an absolute constant (\(K= \sqrt{6+2\sqrt{5}}/3\)).
\end{theorem}

  We now may apply Theorem~\ref{theo:main} to the situation at hand and to do so we  need  the contractivity of the scheme in the \(P\)--norm  studied in Theorem~\ref{theo:thursday}.
The discussion that follows may immediately be extended to all choices of \(c=c(L,m)\) that lead to contractivity; however, for the sake of clarity, we fix \(c=1/L\) as in \cite{CCB18}. (But recall that \(c=1/L\) is suboptimal in terms of the contraction rate). For this choice, we know from Remark~\ref{rem:friday} that, for \(h\leq h_0\), the numerical rate \((1-\rho_h^{1/2})/h\) will  be of the form \(\bar r/\kappa\), where \(\bar r<1/2\) may be chosen arbitrarily close to \(1/2\) (\(h_0\) depends of course on \(\bar r\) but it is independent of \(d\), \(L\) and \(m\)).
Theorem~\ref{theo:main} yields, for \(h\leq h_0\),
\begin{equation}\label{eq:sunday11}
W_P(\pi^\star,\Psi_{n,h}\pi)\leq \left(1-\frac{\bar r h}{\kappa}\right)^n W_P(\pi^\star,\pi)+ \frac{K  \kappa d^{1/2} }{\bar {r} L^{1/2}}h.
\end{equation}

Assume now that, given any initial distribution \(\pi\) for the integrator and \(\epsilon >0\), we wish to find  \(h\) and \(n\) to guarantee that
\(W_P(\pi^\star,\Psi_{n,h}\pi)<\epsilon\). We may achieve this goal by first choosing \(h\leq h_0\) small enough to ensure that
\[
\frac{K  \kappa d^{1/2} }{\bar {r} L^{1/2}}h<\frac{\epsilon}{2}
\]
and then choosing \(n\) large enough to get
\[
\left(1-\frac{\bar r h}{\kappa}\right)^n W_P(\pi^\star,\pi)<\frac{\epsilon}{2}.
\]
(Instead of splitting the target \(\epsilon\) as \(\epsilon/2+\epsilon/2\), one may use \(a\epsilon+(1-a)\epsilon\), \(a\in(0,1)\), and tune \(a\) to improve slightly some of the error constants below.)
This leads to the conditions
\begin{equation}\label{eq:hestimatecheng}
h < \min\left(h_0,\frac{\bar r}{2K} (m^{1/2} \epsilon) \kappa^{-1/2} d^{-1/2}\right),
\end{equation}
(\( m^{1/2}\epsilon\) is a nondimensional combination whose value does not change if the scale of \(x\) is changed)
and
\begin{equation} \label{eq:nestimatecheng}
n > \frac{-1}{\log(1-\bar rh/ \kappa) }\log \frac{2W_P(\pi^\star,\pi)}{\epsilon}.
\end{equation}

According to the bound \eqref{eq:hestimatecheng}, \(h\) has to be scaled  like \((m^{1/2}\epsilon)\kappa^{-1/2} d^{-1/2}\) as \(\epsilon\downarrow 0\), \(\kappa\uparrow \infty\) or \(d\uparrow \infty\);
then the number of steps to achieve a target value of the contraction factor \((1-\bar r h/\kappa)^n\)
 scales as
  \begin{equation}
   \label{eq:n}
   (m^{1/2}\epsilon)^{-1}\kappa^{3/2}d^{1/2}.
   \end{equation}

   The analysis of the same integrator in \cite[Theorem 1]{CCB18} yields  scalings that are more pessimistic in their dependence on \(\kappa\): there, \(h\)  scales as  \((m^{1/2}\epsilon)\kappa^{-1} d^{-1/2}\) and \(n\) as \((m^{1/2}\epsilon)^{-1}\kappa^{2}d^{1/2}\). In addition, the initial distribution \(\pi\), which is arbitrary in the present study, is assumed in \cite{CCB18} to be a Dirac delta located at \(v=0\) and \(x=x_0\); the estimates become worse as the distance between the initial position \(x_0\) used in the integrator   and the mode of \(\exp(-f(x))\) increases.

 It is perhaps useful to compare the technique of proof in \cite{CCB18} with our approach by means of Figure~\ref{fig:schem}. While we employ the contractivity of the algorithm, \cite{CCB18} leverages the contractivity of the SDE itself. These two alternative approaches are well known in deteministic numerical differential equations (see e.g.\ the discussion in \cite[Chapter 2]{HNW00} where a cartoon similar to Figure~\ref{fig:schem} is presented).
 On the other hand, while we investigate \(\phi_h(\cdot,t_n)-\psi_h(\cdot,t_n)\) evaluated at a random variable \(\widehat \xi_n\) whose marginal distribution is \(\pi^\star\), \cite{CCB18} has to evaluate that difference at the numerical \(\xi_n\). It is for this reason that in \cite{CCB18} one needs to  have information on the distribution \(\pi\) of \(\xi_0\) and to establish a priori bounds on the distributions of \(\xi_n\) as \(n\) varies (this is done in Lemma 12 in \cite{CCB18}). Generally speaking, once contractivity estimates are available for the numerical solution, the approach on the left of Figure~\ref{fig:schem} is to be preferred.

 The reference \cite{DD20}  also investigates Wasserstein error bounds for the integrator EE. The general approach is the same as  that in \cite{CCB18}, but the technical details differ.  For  \(c\leq 4/(L+m)\),\footnote{Because \cite{DD20}  uses the SDE contractivity, it does not exclude the limit case \(c=4/(L+m)\) as we have to do. See Remark~\ref{rem:sunday} that implies that for this integrator bounds that hold for \(c=4/(L+m)\) are only possible for \(h=\mathcal{O}(\kappa^{-1})\) .} an upper bound very similar to \eqref{eq:sunday11} is derived for the \(2\)-Wasserstein distance between the \(x\)-marginals of \(\pi^\star\) and \(\Psi_{n,h}\pi\). That bound depends on \(\kappa\), \(L\), \(d\) and \(h\) in the same way as \eqref{eq:sunday11} does. The constants in the estimates are nevertheless different, as expected.
  For instance, for the choice \(c=1/L\) we have been discussing, the factor
 \(1-\bar{r}h/\kappa\) in \eqref{eq:sunday11} (where \(\bar r\) is arbitrarily close to \(1/2\)) is replaced in \cite{DD20} by the slightly worse factor \(1-0.375h/\kappa\). However the bound in \cite{DD20} is only proved for very small values of \(h\) (\(h \leq 1/(8\kappa)\) when \(c=1/L\)). {\color{black} This is an extremely severe limitation because we know from \eqref{eq:hestimatecheng} that, as the condition number increases, EE may be operated with a value of \(h\) of the order of  \(1/\sqrt{\kappa}\) rather than \(1/\kappa\).}
  The unwelcome step size restriction  originates from   estimating  \(\phi_h(\cdot,t_n)-\psi_h(\cdot,t_n)\) evaluated at the \emph{numerical} solution rather than at the SDE solution.

\subsection{The UBU integrator}
\label{ss:ubujuly}
We now turn our attention to the UBU integrator. Under the standard smoothness Assumptions \ref{as1} and \ref{as2}, the strong order of convergence of UBU is \(p=1\) {\color{black}(see Section~\ref{sec:LEubu2} for a detailed analysis of the UBU local error under those assumptions). The  analysis for UBU is then very similar to the one presented above for EE, and leads to an \(\mathcal{O}(\epsilon^{-1} \kappa^{3/2} d^{1/2})\) estimate for the mixing time.}
When \(f\) satisfies the additional smoothness Assumption \ref{as3}, UBU exhibits strong order \(p=2\) and this may be used in our context to improve on the estimates \eqref{eq:hestimatecheng}--\eqref{eq:nestimatecheng}.

The proof of the next result is given in Section~\ref{sec:LEubu}.

\begin{theorem}
\label{theo:ubu}
Assume that \(f\) satisfies Assumptions \ref{as1}--\ref{as3}.
 Set  \(\gamma = 2\) and \(P_h=\wP\otimes I_d\) with \(\wP\) as in \eqref{eq:Pforunderdamped}.
Then, for \(h\leq 2\), the discretization \eqref{eq:cheng} satisfies Assumption~\ref{as4} with \(p=2\),
\begin{eqnarray*}
C_0 & = & K_0(2+cL),\\
C_1 &=& K_1c^{3/2} L d^{1/2},\\
C_2 & = & K_2\Big((1+4\sqrt{3}) c^2L^{3/2}+(3+\frac{\sqrt{42}}{2}) c^{3/2}L+ 6cL^{1/2}+\sqrt{3}c^2L_1\Big) d^{1/2},
\end{eqnarray*}
where \(K_j\), \(j=0,1,2,\) are the following absolute constants
\[
K_0 = \sqrt{\frac{2\sqrt{2}}{3-\sqrt{5}}},
\qquad K_1 =\frac{\sqrt{3}}{12}, \qquad K_2 = \frac{1}{24} \sqrt{\frac{3+\sqrt{5}}{2}}.
\]
\end{theorem}

The contractivity of the scheme in the \(P\)-norm necessary to use Theorem~\ref{theo:main} was  established in Theorem~\ref{theo:thursday}.
As we did for the first-order integrator, for the sake of clarity, we fix \(c = 1/L\) (but other values of \(c\) may be discussed similarly, provided that they ensure the contractivity of the algorithm). Note that the constant \(C_0=K_0(2+cL)\) is then \(\leq 3K_0\). After choosing \(\bar r<1/2\) arbitrarily as in Remark~\ref{rem:friday}, Theorem~\ref{theo:main} yields, for \(h\leq h_0\):
\begin{equation}\label{eq:monday12}
W_P(\pi^\star,\Psi_{n,h}\pi)\leq \left(1-\frac{\bar r h}{\kappa}\right)^n W_P(\pi^\star,\pi)+ \bar{K} \left(\frac{1}{\sqrt{L}}+\frac{L_1}{L^2}\right)\kappa d^{1/2}h^2,
\end{equation}
where  \(\bar K\) denotes an absolute constant. To ensure \(W_P(\pi^\star,\Psi_{n,h}\pi)<\epsilon\), we take
\[\bar{K} \left(\frac{1}{\sqrt{L}}+\frac{L_1}{L^2}\right)\kappa d^{1/2}h^2<\frac{\epsilon}{2},\]
and then increase  \(n\) as in \eqref{eq:nestimatecheng}. Thus, for UBU, the scaling of \(h\) is
\[ (m^{1/2}\epsilon)^{1/2}\kappa^{-1/4}(1+L^{-3/2}L_1)^{-1/2}d^{-1/4},
\]
and, as a consequence, the number of steps \(n\) to guarantee a target contraction factor \((1-\bar r h/\kappa)^n\) scales as
\begin{equation}\label{eq:n2}
(m^{1/2}\epsilon)^{-1/2}\kappa^{5/4}(1+L^{-3/2}L_1)^{1/2}d^{1/4}.
\end{equation}
The dependence of \(n\) on \(m^{1/2}\epsilon\), \(\kappa\) and \(d\) in this estimate is far more favourable than it was for {\color{black} EE} (see \eqref{eq:n}).
However here we have the (\(L_1\) dependent) factor \((1+L^{-3/2}L_1)^{1/2}\) and one could easily concoct examples of distributions where this factor is large even if the condition number is of moderate size. In those, arguably artificial, particular cases, it may be advantageous to see UBU as a first order method {\color{black} as discussed at the beginning of this section.}
{\color{black}
\begin{rem}
\label{rem:secondorder}A comparison between \eqref{eq:sunday11} and \eqref{eq:monday12} makes it clear that (for fixed \(m\), \(L\), \(L_1\)) the \(\epsilon^{-1/2} d^{1/4}\) dependence of the mixing time of  UBU  stems from having strong order two. In the second term of the right hand-side of the inequalities  \eqref{eq:sunday11} and \eqref{eq:monday12} (i.e.\ the bias), the exponent of \(h\) coincides with the strong order of the integrator. In order to make those second terms of size $\epsilon$ one needs to scale \(h\) as
$\epsilon d^{-1/2}$ for EE and as $\epsilon^{1/2} d^{-1/4}$ for UBU.
The first terms of the right hand-side of the inequalities  \eqref{eq:sunday11} and \eqref{eq:monday12}, then show that \(n\) has to be scaled as \(h^{-1}\), i.e.\ as \(\epsilon^{-1} d^{1/2}\) for the first-order method and \(\epsilon^{-1/2} d^{1/4}\) for the second-order method.

Integrators of strong order higher than two would have even more favourable  dependence of the mixing time on $\epsilon$ and  $d$. Unfortunately such high-order integrators \cite{MT07} are invariably too complicated to be of much practical significance. In particular there is no splitting algorithm that achieves strong order larger than two \cite{AS19}. In addition, an increase of the order may be expected to require an increase of the required smoothness of $f$.

The randomized algorithm in
\cite{ShenLee2019} has a bias that behaves as $d^{1/2}h^{3/2}$ leading to an $\epsilon^{-2/3} d^{1/3}$ estimate of the mixing time.
\end{rem}
}

\begin{rem}
The paper \cite{DK19}  considers a weaker form of the extra-smoothness assumption Assumption \ref{as3} where \(\calH(x)\), rather than assumed to be differentiable with derivative upper-bounded by \(L_1\),  is only assumed to be Lipschitz continuous with constant \(L_1\). It is likely that, by means of the technique in the proof of \cite[Lemma 6]{DK19},  Theorem  \ref{theo:ubu} may be proved under that alternative, weaker version of Assumption \ref{as3}, but we have not yet studied that possibility.
\end{rem}

A second order discretization of the underdamped Langevin equation, that unlike UBU, requires to evaluate the Hessian of \(f\) once per step, has been suggested  in \cite{DD20}. A bound similar to \eqref{eq:monday12} is derived which is valid only for small values \(h\leq \mathcal{O}(\kappa^{-1}) \wedge \mathcal{O}( L^{1/2} m d^{-1/2} L_1^{-1})\). {\color{black} This is very restritive because, as we have just found, for UBU \(h\) scales with \(\kappa\) as \(\kappa^{-1/4}\). }

The reference \cite{FLO21} suggests a novel approach to obtaining high order discretizations of
the underdamped Langevin dynamics \eqref{eq:underdamped}. At each step, in addition to generating suitable random variables, one has to integrate a so-called \lq\lq shifted ODE\rq\rq, whose solutions are smoother than the solutions of \eqref{eq:underdamped}. The analysis in that reference examines the case where the integration is exact; in practice, the shifted ODE has of course to be discretized by a suitable numerical method and \cite{FLO21} provides numerical examples based on two different choices of such a method.

\section{Additional results and proofs}
\label{sec:final}
\subsection{Proof of Proposition~\ref{prop:connection}}
\label{sec:proofprop1}
The second item in the Proposition {\color{black} is proved as follows. By standard linear algebra results, \(\xi\) may be uniquely decomposed
as \(\xi = n+ m\) with \(n\) in the kernel of \(C\) and \(m\) in the image of \(C^T\); furthermore there is a bijection between values of \(m\) and values of \(x =C\xi\). Under the assumption \(SC^T= 0\), which implies \(CS=0\), \(Sm\) and \(m^TS\) vanish, and then \(\xi^TS\xi = n^TSn\) is independent of \(m\), i.e.\ of \(x\). Therefore the marginal of \(\propto \exp(-f(C\xi)-(1/2) \xi S\xi)\) coincides with the marginal of
\(\propto \exp(-f(C\xi))\) .}

For the first {\color{black} item},
we have to show that the pdf \(k \exp\big(-f(x)-(1/2) \xi^T S\xi\big)\) (\(k\) is the normalizing constant), that with some abuse of notation we denote by \(\pi^\star\) satisfies the Fokker-Planck equation
\[
-\nabla_\xi\cdot \big( \pi^\star (A\xi + B\nabla f(x))\big)+\nabla_\xi\cdot (D \nabla_\xi \pi^\star) = 0.
\]
Here \(\nabla_\xi\cdot\) and \(\nabla_\xi\) respectively denote the standard divergence and gradient operators in the space \(\R^N\) of the variable \(\xi\). The computations that follow use repeatedly the well-known identity \(\nabla_\xi\cdot (c F) = c \,\nabla_\xi\cdot F+F^T\nabla_\xi c\), where \(c=c(\xi)\) is a scalar valued function  and \(F = F(\xi)\) is an \(\R^N\)-valued function. We will also use that if
\(R\) is any \(M\times d\) constant matrix, then \(\nabla_\xi\cdot (R \nabla f(x))= CR:H(x)\), where \(H(x)\) denotes the Hessian of \(f(x)\) and \(:\) stands for the Frobenius product of matrices (equivalently \(CR:H(x)={\rm Tr}((CR)^TH(x)\)).

We observe that
\[
\nabla_\xi \pi^\star = -\pi^\star \big(S\xi+C^T\nabla f(x)\big)
\]
and therefore
\[
\nabla_\xi\cdot( \pi^\star A\xi)=\pi^\star {\rm Tr}(A) -\pi^\star \xi^TA^T S\xi-\pi^\star \xi^TA^TC^T\nabla f(x)
\]
and
\begin{eqnarray*}
\nabla_\xi\cdot \big(\pi^\star B\nabla f(x)\big)&=&\pi^\star (CB:H(x))\\ && -\pi^\star (\nabla f(x))^TB^TS\xi-\pi^\star (\nabla f(x))^TB^TC^T\nabla f(x).
\end{eqnarray*}
Furthermore
\begin{eqnarray*}
\nabla_\xi\cdot (D \nabla_\xi \pi^\star)  &=& -\nabla_\xi\cdot (\pi^\star D S\xi)-\nabla_\xi\cdot(\pi^\star DC^T\nabla f(x))\\
&=& -\pi^\star {\rm Tr}(DS)+\pi^\star \xi^TSDS\xi+\pi^\star \xi^TSDC^T\nabla f(x)\\
&&-\pi^\star (CDC^T:H(x))\\
&&+\pi^\star (\nabla f(x))^T CD S\xi+\pi^\star (\nabla f(x))^TCDC^T\nabla f(x).
\end{eqnarray*}

From the last three displays we conclude that the left hand-side of the Fokker-Planck equations  is the product of \(\pi^\star\) and
\begin{eqnarray*}
&&-{\rm Tr} (A+DS)\\&&-(CB+CDC^T):H(x)+(\nabla f(x))^T (B^TC^T+CDC^T) \nabla f(x)\\
&& + \xi^T(A^TC^T+SB+2SDC^T)\nabla f(x)\\
&& +\xi^T(A^TS+SDS)\xi;
\end{eqnarray*}
 each of the first three relations in \eqref{eq:relations} is sufficient for the corresponding line in this display to vanish. (In addition, if \(f\) is regarded as arbitrary, then those three relations are also \emph{necessary}.) The quadratic form in the fourth line in the display vanishes if and only if \(A^TS+SDS\) is skew-symmetric as demanded by the fourth relation in \eqref{eq:relations}. This completes the proof.

\subsection{Contraction estimates for the underdamped Langevin equations}
\label{secc:decay}
We consider the underdamped Langevin equations \eqref{eq:underdamped} where, after rescaling \(t\), we may assume that \(\gamma = 2\).
We apply Proposition~\ref{prop:LZ} to  determine \(\wP\) and \(c\) so as to  maximize the decay rate  \(\lambda\).
We exclude the case \(L=m\), which has no practical relevance.

 If
\[
\widehat L = \left[\begin{matrix} \ell_{11}& 0\\ \ell_{12}& \ell_{22}\end{matrix}\right],
\]
\(\ell_{11},\ell_{22} >0\), denotes the unknown Choleski factor of \(\wP\),  the eigenvalues of \(\wL^{-1}\wZ_i\wL^{-T}\) are found to be
\begin{equation}\label{eq:eigappendix}
2 \pm \frac{\sqrt{\left(\ell^{2}_{11} cH_i -2\ell_{11}\ell_{21}+\ell^{2}_{21}-\ell^{2}_{22} \right)^{2}+4\ell^{2}_{22}\left(\ell_{11}-\ell_{21} \right)^{2}}}{\ell_{11}\ell_{22}}.
\end{equation}
  Since our aim is to ensure that these have a lower bound as large as possible and  we only have to consider the minus sign in \eqref{eq:eigappendix}.

   Without loss of generality, we may set \(\ell_{11} = 1\) and then have to find \(c>0\), \(\ell_{22}>0\) and \(\ell_{21}\) to minimize
\begin{equation}\label{eq:sup}
\sup_{m\leq H\leq L} \left\{\frac{1}{\ell_{22}^2}\Big[cH-(\ell_{22}^2+2\ell_{21}-\ell_{21}^2)\Big]^2+4(1-\ell_{21})^2\right\}.
\end{equation}
Consider a local minimum \(c\),  \(\ell_{21}\), \(\ell_{22}\) of the minimization problem. We claim that
\begin{equation}\label{eq:midpoint}
c \frac{L+m}{2} = a
\end{equation}
where
\[
a= \ell_{22}^2+2\ell_{21}-\ell_{21}^2.
\]
In other words \(a\) has to coincide with the midpoint of the interval \([cm,cL]\) of possible values of \(cH\).
In fact, assume that \(c (L+m)/2 > a\), i.e.\ the point \(a\) is to the left of the midpoint. Then the supremum in \eqref{eq:sup} is attained at \(H = L\), because
\(|cm-a| <|cL-a|\); we could lower the value of the supremum by decreasing slightly \(c\). If \(u (L+m)/2 < a\)
the supremum decreases by increasing slightly \(c\).

When \eqref{eq:midpoint} holds the supremum  in \eqref{eq:sup} is the common value that the expression in braces takes at \(H = m\) and \(H=L\). This common value is:
\[
 \frac{1}{\ell_{22}^2}\left(\frac{L-m}{L+m}\right)^2(\ell_{22}^2+2\ell_{21}-\ell_{21}^2)^2+4(1-\ell_{21})^2,
\]
that we rewrite as
\[
\frac{1}{\ell_{22}^2}\left(\frac{L-m}{L+m}\right)^2\big(\ell_{22}^2+1-(1-\ell_{21})^2\big)^2+4(1-\ell_{21})^2.
\]
With \(b = (L-m)^2/(L+m)^2<1\), our task is  to find \(X = \ell_{22}^2>0\), \(Y=(1-\ell_{21})^2\geq 0\) so as to minimize
\[
F(X,Y) = \frac{b}{X}  (X+1-Y)^2+4Y.
\]

We first fix \(Y\in [0,1)\). Then \(F\rightarrow  \infty\) as  \(X\rightarrow 0\) or \(X\rightarrow \infty\). By setting \(\partial F/\partial X = 0\),  we easily see that \(F\) has a unique minimum at \(X = 1-Y\in (0,1]\). At that minimum
\[
F = 4(b-1)X+4,
\]
which, in turn, is minimized by taking \(X\) as large as possible, i.e. \(X=1\). Then \(Y=0\) and \(F = 4b < 4\).
We then fix \(Y \geq 1\). In this case, \(F \geq 4Y\geq 4\), which is worse than the best \(F\) that can be achieved with \(Y\in [0,1)\).
To sum up: the optimum value of \(F\) is \(4b\) and is achieved when \(X=1\), \(Y=0\), i.e. \(\ell_{22} = 1\), \(\ell_{21} = 1\); then the matrix \(\wP\) is the one in \eqref{eq:Pforunderdamped}. Taking these values of \(\ell_{ij}\) to \eqref{eq:midpoint}, we find that \emph{ \(c = 4/(L+m)\) provides the optimal parameter choice}. Finally, since the best value of \eqref{eq:sup} is \(4b\), the expression \eqref{eq:eigappendix} shows that \emph{the best decay rate is \(\lambda=2-\sqrt{4b} = 4m/(L+m)\)}.
\subsection{Integrators for the underdamped Langevin equations}
\label{sec:UBU}
Due to the importance of \eqref{eq:underdamped} in statistical physics and molecular dynamics, the literature on its numerical integration is by now enormous. It is completely out of our scope to summarize it and we limit ourselves to a few comments on splitting algorithms.

The different terms in the right hand-side of \eqref{eq:underdampedv} and \eqref{eq:underdampedx} correspond to different, separate physical effects, like inertia, noise, damping, etc. Therefore the system \eqref{eq:underdamped} it is ideally suited to splitting algorithms \cite{MQ02}. A possible way of carrying out the splitting is
\begin{align*}
 {\rm (A)} &\qquad (d/dt)v = 0,\: &(d/dt)x = v,\\
 {\rm (B)} & \qquad (d/dt)v = -c\nabla f(x),\: &(d/dt) x = 0,\\
 {\rm (O)} &\qquad  d v = -\gamma vdt+ \sqrt{2\gamma c}dW,\: &(d/dt) x = 0.
\end{align*}
Each of these subsystems may be integrated in closed form. This partitioning gives rise to schemes like ABOBA, BAOAB and OBABO \cite{LM13,LM15,M20}. For instance, a step of ABOBA  first advances the numerical solution over \([t_n,t_{n+1/2}]\) using the exact solution of (A), then over \([t_n,t_{n+1/2}]\) using the exact solution of (B), then over \([t_n,t_{n+1}]\) with (O), then over \([t_{n+1/2},t_{n+1}]\) with (B) and finally closes symmetrically with (A) over \([t_{n+1/2},t_{n+1}]\). ABOBA, BAOAB and OBABO have weak order two but only possess strong order one. In fact it is easy to check that with the A-B-O splitting it is impossible to generate the Ito integrals that are required to ensure strong order two.

Another subsystem that may be integrated exactly in closed form is
\begin{align*}
 {\rm (U)} &\qquad  d v = -\gamma vdt+\sqrt{2\gamma c} dW,\: &(d/dt) x = v.
\end{align*}
The algorithm BUB, used e.g. in \cite{BTT20}, advances with (B) over \([t_n,t_{n+1/2}]\), then with (U) over \([t_n,t_{n+1}]\) and closes the step with (B) over \([t_{n+1/2},t_{n+1}]\). To our best knowledge it was first suggested by Skeel \cite{S19}. In \cite{FLO21} is referred to as \emph{Strang splitting}, but this terminology may be confusing because it is standard to use the expression Strang splitting to mean any splitting algorithms with a symmetric pattern XYX \cite{sanz2018numerical}. The authors of \cite{FLO21}, a reference that compares different integrators,  \lq\lq believe that BUB offers an attractive compromise between accuracy and computational cost\rq\rq.

Changing the roles of B and U we obtain the UBU scheme suggested in the thesis \cite{AZ19}, where it is proved that both BUB and UBU have weak and strong order two. In fact BUB and UBU are closely related because, UBU is the algorithm that advances the BUB solution from the midpoint \(t_{n+1/2}\) of one step to the midpoint \(t_{n+3/2}\) of the next (and vice versa).

The thesis \cite{AZ19} also describes a method to boost to  strong order two any method whose strong order is only one. The boosting is achieved by generating auxiliary Gaussian random variables and may be relevant in our context where we are interested in the strong order of the integrators. A general technique to investigate the weak and strong order of splitting algorithms for general SDEs  may be seen in \cite{AS16,AS19}; those reference provide a detailed study of splitting Langevin integrators.

\subsection{A lemma}
\label{sec:lemma}
The  following result is  a variant of \cite[Lemma 7]{DK19} and may be proved in a similar way.
\begin{lemma}Assume that the sequence of nonnegative numbers \((z_n)\) is such that for some constants \(A\in(0,1)\), \(B\geq 0\), \(C\geq 0\) and each \(n=0,1,2,\dots\)
\[
z_{n+1} \leq \sqrt{(1-A)^2 z_n^2+B}+C.
\]
Then, for \(n=0,1,2,\dots\)
\[
z_n \leq (1-A)^nz_0+\sqrt{\frac{B}{A}}+\frac{C}{A}.
\]
\end{lemma}
\subsection{The local error for EE integrator}
\label{sec:LEcheng}

To analyze
\(\phi_h(\widehat{\xi}_n,t_n)-\psi_h(\widehat{\xi}_n,t_n)\), we have to take the random variable \(\widehat{\xi}_n=(\widehat{v}_n,\widehat{x}_n) \sim \pi^{*} \)  (see Theorem~\ref{theo:main}) as initial data, first to move the solution of the SDE forward in the interval
\([t_n,t_n+h]\) and then to  perform a step of the integrator over the same interval.
\begin{subequations}
Solutions of \eqref{eq:overdamped} satisfy, for \(t\geq t_n\), \(n=0,1,2,\dots\),
\begin{eqnarray}\label{eq:truev}
 v(t) &=& \calE(t-t_n) v(t_n) -\int_{t_n}^t \calE(t-s)c\nabla f(x(s))\,ds
  +\sqrt{2\gamma c} \int_{t_n}^t \calE(t-s)\,dW(s),\\
  \label{eq:truex}
 x(t) & = & x(t_n) +\calF(t-t_n)v(t_n)
-\int_{t_n}^t\calF(t-s)c\nabla f(x(s))\,ds
 +\sqrt{2\gamma c} \int_{t_n}^t \calF(t-s)\,dW(s),
\end{eqnarray}
\end{subequations}
The proof is divided up in steps.

\emph{First step. } From \eqref{eq:truev} and \eqref{eq:chengv}, we find that the \(v\)-component of \(\phi_h(\widehat{\xi}_n,t_n)-\psi_h(\widehat{\xi}_n,t_n)\), denoted as \(\Delta_v\), is
\[
\Delta_v = -\int_{t_n}^{t_{n+1}} \calE(t_{n+1}-s)c \big(\nabla f(x(s))-\nabla f(\widehat{x}_n)\big) \, ds;
\]
 where \(x(s)\) is the \(x\)-component of  the solution of \eqref{eq:underdamped} that at \(t_n\) takes the value \(\widehat{\xi}_n\) (we shall later need the \(v\) component, to be denoted by \(v(s)\)).  Using successively the Cauchy-Schwartz inequality, the bound \(\calE(t)\leq 1\) for \(t\geq 0\), the Lipschitz continuity of
\(\nabla f(x)\), and \eqref{eq:underdampedx}, we find:
\begin{eqnarray*}
\E\big(\|\Delta_v\|^2 \big) & \leq  & c^2\E\left(\int_{t_n}^{t_{n+1}} \calE(t_{n+1}-s)^2 ds\: \int_{t_n}^{t_{n+1}}  \|\nabla f(x(s))-\nabla f(\widehat{x}_n)\|^2 ds\right)\\
&\leq&
hc^2 \int_{t_n}^{t_{n+1}}\E\big(\|\nabla f(x(s))-\nabla f(\widehat{x}_n)\|^2\big) ds\\
&\leq& hc^2L^2 \int_{t_n}^{t_{n+1}}\E\big(\|x(s)-\widehat{x}_n\|^2\big) ds\\
&\leq& hc^2L^2 \int_{t_n}^{t_{n+1}}\E\big(\|\int_{t_n}^s v(s^\prime)ds^\prime\|^2\big) ds.
\end{eqnarray*}
An application of the Cauchy-Schwartz inequality to the inner integral  yields
\[
\E\big(\|\Delta_v\|^2 \big)  \leq hc^2L^2 \int_{t_n}^{t_{n+1}}s \left(\int_{t_n}^{s}
\E\big( \|v(s^\prime)\|^2\big) ds^\prime\right) ds.
\]
Now, using the fact that the initial data $\widehat{\xi}_{n}$ is distributed according to $\pi^{*}$, this will be the distribution of $v(s^\prime)$ for all $s^{\prime}\geq t_n$. Hence, since  the distribution of each of the \(d\) scalar components of \(v\) is centered Gaussian with second moment equal to \(c\), we obtain the final bound
\[
\E\big(\|\Delta_v\|^2 \big)  \leq hc^2L^2d \int_{t_n}^{t_{n+1}}s \left(\int_{t_n}^{s}c\,
 ds^\prime\right) ds = \frac{h^4}{3} c^3L^2d.
\]

\emph{Second step.} Turning now to the \(x\) component of
\(\phi_h(\widehat{\xi}_n,t_n)-\psi_h(\widehat{\xi}_n,t_n)\), we have
\[
\Delta_x = -\int_{t_n}^{t_{n+1}} \calF(t_{n+1}-s)c \big(\nabla f(x(s))-\nabla f(\widehat{x}_n)\big) \, ds,
\]
and, applying the Cauchy-Schwartz inequality and the bound \(\calF(t)\leq t\),
\begin{eqnarray*}\E\big(\|\Delta_x\|^2 \big)  &\leq& c^2 \E
\left(\int_{t_n}^{t_{n+1}} \calF(t_{n+1}-s)^2 ds
\int_{t_n}^{t_{n+1}}   \|\nabla f(x(s))-\nabla f(\widehat{x}_n)\|^2\right) ds\\
&\leq &\frac{h^3}{3} c^2\int_{t_n}^{t_{n+1}}   \E\big(\|\nabla f(x(s))-\nabla f(\widehat{x}_n)\|^2ds\big) .
\end{eqnarray*}
Therefore, by proceeding in the last integral as we when we found it above, we find
\[
\E\big(\|\Delta_x\|^2 \big) \leq \frac{h^6}{9} c^3L^2d.
\]

\emph{Third step.} The preceding analysis is valid for all values of the parameters. We now assume that \(t\) is measured in units for which \(\gamma=2\) and that \(h\) is chosen \(\leq 1\). Then,
combining the estimates for the \(v\) and \(x\) components:
\begin{equation} \label{eq:mondayjuly}
\E\big(\|\phi_h(\widehat{\xi}_n,t_n)-\psi_h(\widehat{\xi}_n,t_n)\|^2 \big)\leq
\frac{h^4}{3} c^3L^2d+ \frac{h^6}{9} c^3L^2d\leq \frac{4}{9}h^4 c^3L^2d.
\end{equation}
For the matrix \(\wP\) in \eqref{eq:Pforunderdamped} the constant \(p_{\max}\) in \eqref{eq:sandwich}
is easily computed as \((3+\sqrt{5})/2\) and we finally have:
\[\|\phi_h(\widehat{\xi}_n,t_n)-\psi_h(\widehat{\xi}_n,t_n)\|^2_{L^2,P} \leq \frac{6+2\sqrt{5}}{9} h^4c^3L^2d.
\]

\subsection{The local error for UBU}
\label{sec:LEubu}

{\color{black}We first provide bounds for the local error for  UBU under Assumptions \ref{as1}--\ref{as3} that ensure strong order two.}
As in the previous Subsection we have to take \((\widehat{v}_n,\widehat{x}_n)\) as a  starting point for the SDE solution  and the integrator. As with the {\color{black} EE} integrator, \(v(t)\) and \(x(t)\) denote the solution of \eqref{eq:underdamped} that starts at \(t_n\) from \((\widehat{v}_n,\widehat{x}_n)\).
The analysis is now substantially more involved as the Ito-Taylor expansions have to be taken to higher order.

\emph{First step.} We begin by estimating the difference \(\Delta_y\) between \(x(t_n+h/2)\) and the point \(y_n\) where the integrator evaluates the force \(-\nabla f\) (see \eqref{eq:ubuy}). By using \eqref{eq:truex} and \eqref{eq:ubuy},
we find
\[
\Delta_y = x(t_n+h/2)-y_n = -\int_{t_n}^{t_{n+1/2}} \calF(t_{n+1/2}-s) c\nabla f(x(s))\, ds.
\]
We apply the Cauchy-Schwartz inequality (in a similar way to what we did at Step 2 in the preceding subsection) to get
\[
\E\big (\| \Delta_y \|^2\big) \leq \frac{h^3}{24}c^2\int_{t_n}^{t_{n+1/2}}\E\big(\| \nabla f(x(s)\|^2\big)\,ds.
\]
As proved in \cite[Lemma 2]{D17a}, when \(\bar x\) has the distribution \(\pi^\star\),
\begin{equation}\label{eq:dalalyanlemma}
\E\big(\|\nabla f (\bar x)\|^2 \big) \leq Ld
\end{equation}
and, accordingly,
\begin{equation}\label{eq:ububounddeltay}
\E\big (\| \Delta_y \|^2\big) \leq \frac{h^4}{48}c^2Ld.
\end{equation}

\emph{Second step.} From \eqref{eq:truev} and \eqref{eq:ubuv}, the \(v\)-component of \(\phi_h(\widehat{\xi}_n,t_n)-\psi_h(\widehat{\xi}_n,t_n)\) is found to be
\begin{equation}\label{eq:deltavubu}
\Delta_v = - \int_{t_n}^{t_{n+1}} \calE(t_{n+1}-s) c \nabla f(x(s))\,ds + h  \calE(h/2) c\nabla f(y_n);
\end{equation}
thus UBU replaces the integral by a midpoint-rule approximation. We  Ito-Taylor expand (see e.g.\ \cite{KP92,AS19}) the integral  around \(t_{n+1/2}\) as follows. Denote by \(\chi(s)\) the (differentiable) integrand, i.e.
\[
\chi(s) = \calE(t_{n+1}-s) c \nabla f(x(s)).
\]
Then (the dot indicates differentiation),
\begin{equation}\label{eq:july17}
\int_{t_n}^{t_{n+1}} \chi(s) ds = \int_{t_n}^{t_{n+1}}  \chi(t_{n+1/2})ds + \int_{t_n}^{t_{n+1}} ds \int_{t_{n+1/2}}^s  \dot{\chi}(s)ds^\prime,
\end{equation}
{\color{black} and taking the expansion one step further, we find}
\begin{eqnarray}\nonumber
\int_{t_n}^{t_{n+1}} \chi(s) ds &= & h\chi(t_{n+1/2})+\int_{t_{n+1/2}}^{t_{n+1}} ds
\int_{t_{n+1/2}}^s  \big( \dot{\chi}(s^\prime)- \dot{\chi}(2t_{n+1/2}-s^\prime)\big)ds^\prime\\
\label{eq:july17bis}
&=& h\chi(t_{n+1/2})+\int_{t_{n+1/2}}^{t_{n+1}} ds
\int_{t_{n+1/2}}^s ds^\prime \int_{2t_{n+1/2}-s^\prime}^{s^{\prime}} d\dot{\chi}(s^{\prime\prime}).
\end{eqnarray}
We now  replace  \(d\dot{\chi}(s^{\prime\prime})\) by its expression given by Ito's lemma.
(While \(\chi\) is differentiable, \(\dot \chi\) is a diffusion process.) There is no Ito correction because \(\dot \chi\) is linear in \(v\) and there is no forcing noise in the \(x\) equation \eqref{eq:underdampedx}.
 After computing \(d\dot{\chi}(s^{\prime\prime})\) and substituting back in \eqref{eq:deltavubu}, we have
\begin{equation}\label{eq:ubudeltavbis}
\Delta_v  = -h c \calE(h/2) \big(\nabla f(x(t_{n+1/2}))-\nabla f(y_n)\big)+I_1+I_2+I_3+I_4 +I_5,
\end{equation}
with
\begin{eqnarray*}
I_1 &=& -\int_{t_{n+1/2}}^{t_{n+1}} ds \int_{t_{n+1/2}}^s ds^\prime \int_{2t_{n+1/2}-s^\prime}^{s^{\prime}}\calE(t_{n+1}-s^{\prime\prime})
\gamma^2 c \nabla f (x(s^{\prime\prime}))ds^{\prime\prime},\\
I_2 &=& -\int_{t_{n+1/2}}^{t_{n+1}} ds \int_{t_{n+1/2}}^s ds^\prime \int_{2t_{n+1/2}-s^\prime}^{s^{\prime}}\calE(t_{n+1}-s^{\prime\prime})
\gamma c \calH(x(s^{\prime\prime}))v(s^{\prime\prime})ds^{\prime\prime},\\
I_3 &=& -\int_{t_{n+1/2}}^{t_{n+1}}\!\!\!\! ds \int_{t_{n+1/2}}^s\!\!\!\!\! ds^\prime \int_{2t_{n+1/2}-s^\prime}^{s^{\prime}}\!\!\!\!\calE(t_{n+1}-s^{\prime\prime})
 c \calH^\prime(x(s^{\prime\prime}))[v(s^{\prime\prime}),v(s^{\prime\prime})]ds^{\prime\prime} ,\\
I_4 &=& \int_{t_{n+1/2}}^{t_{n+1}} ds \int_{t_{n+1/2}}^s ds^\prime \int_{2t_{n+1/2}-s^\prime}^{s^{\prime}}\calE(t_{n+1}-s^{\prime\prime})
 c^2 \calH(x(s^{\prime\prime}))\nabla f(x(s^{\prime\prime}))ds^{\prime\prime},\\
 I_5 &=&- \sqrt{2\gamma c}\int_{t_{n+1/2}}^{t_{n+1}} ds \int_{t_{n+1/2}}^s ds^\prime \int_{2t_{n+1/2}-s^\prime}^{s^{\prime}}\calE(t_{n+1}-s^{\prime\prime})
 c\calH(x(s^{\prime\prime}))dW(s^{\prime\prime}).
\end{eqnarray*}

We now successively bound each term in the right hand-side of \eqref{eq:ubudeltavbis}. From \eqref{eq:ububounddeltay} and the Lipschitz continuity of \(\nabla f(x)\)
\[
\E\Big(\| h c \calE(h/2) \big(\nabla f(x(t_{n+1/2}))-\nabla f(y_n)\big)\|^2\Big) \leq \frac{h^6}{48} c^4L^3d.
\]

The integral  \(I_1\) may be bounded as follows:
\begin{eqnarray*}
\E\big(\|I_1\|^2\big)&\leq&\gamma^4c^2 \E\left[\left(\int_{t_{n+1/2}}^{t_{n+1}} ds \int_{t_{n+1/2}}^s ds^\prime \int_{2t_{n+1/2}-s^\prime}^{s^{\prime}} \calE(t_{n+1}-s^{\prime\prime})^{2}ds^{\prime\prime}\right)\right. \\
&&
\left.\left(
\int_{t_{n+1/2}}^{t_{n+1}} ds \int_{t_{n+1/2}}^s ds^\prime \int_{2t_{n+1/2}-s^\prime}^{s^{\prime}} \|\nabla f (x(s^{\prime\prime}))\|^2ds^{\prime\prime}\right)\right]
\end{eqnarray*}
Now using the fact that $\calE(t) \leq 1$, for $t \geq 0$, we can bound the first term in the equation above by observing that
\[
\int_{t_{n+1/2}}^{t_{n+1}} ds \int_{t_{n+1/2}}^s ds^\prime \int_{2t_{n+1/2}-s^\prime}^{s^{\prime}}ds^{\prime\prime} = \frac{h^3}{24}
\]
and then take into account \eqref{eq:dalalyanlemma},  to get
\[
\E\big(\|I_1\|^2\big)\leq \frac{h^6}{576}\gamma^4 c^2 Ld.
\]

A bound for \(I_2\) may be derived similarly, by using, instead of \eqref{eq:dalalyanlemma},
\[
\E\big(\|\calH(\bar x)\bar v\|^2 \big) \leq L ^2 \E\big(\|\bar v\|^2 \big)=L^2cd.
\]
(\((\bar v,\bar x)\sim \pi^\star\)). Then
\[
\E\big(\|I_2\|^2\big)\leq \frac{h^6}{576}\gamma^2 c^3 L^2d.
\]

For \(I_3\) we use (each scalar component of \(\bar v\) is a centered Gaussian with variance \(c\) and fourth moment \(3c^2\))
\[
\E\big( \| \calH^\prime(\bar x) [\bar v,\bar v]\|^2\big) =  L_1^2
\E\big( \|\bar v\|^4\big)\leq 3L_1^2 c^2d,
\]
that leads to
\[
\E\big(\|I_3\|^2\big)\leq \frac{3h^6}{576} c^4 L_1^2d.
\]

Turning now to \(I_4\),  a new application of \eqref{eq:dalalyanlemma} gives
\[
\E \big(\|\calH(\bar x) \nabla f(\bar x)\|^2 \big) \leq L^2 \E \big(\| \nabla f(\bar x)\|^2\big)
\leq L^3 d,
\]
and then
\[
\E\big(\|I_4\|^2\big)\leq \frac{h^6}{576} c^4 L^3d.
\]

Using  the Cauchy-Schwartz inequality for the inner product associated with the integration on \(s\) and \(s^\prime\), we have
\begin{align*}
&\E\big(\|I_5\|^2\big)\leq 2\gamma c^3\E\left[\left( \int_{t_{n+1/2}}^{t_{n+1}} ds \int_{t_{n+1/2}}^s ds^\prime \right)\right.\times\\
&\left.\left(\int_{t_{n+1/2}}^{t_{n+1}} ds \int_{t_{n+1/2}}^s\left\|\int_{2t_{n+1/2}-s^\prime}^{s^{\prime}}\calE(t_{n+1}-s^{\prime\prime})
 \calH(x(s^{\prime\prime}))dW(s^{\prime\prime})\right\|^2 ds^\prime\right)\right]
\end{align*}
and, with the help of the Ito isommetry, since the Frobenius norm of \(\calH\) is bounded by \(L^2d\),
\[
\E\big(\|I_5\|^2\big)\leq 2\gamma c^3\frac{h^2}{8}
 \int_{t_{n+1/2}}^{t_{n+1}} ds \int_{t_{n+1/2}}^sds^\prime\int_{2t_{n+1/2}-s^\prime}^{s^{\prime}}L^2 d\, ds^{\prime\prime}= \frac{h^5}{96}\gamma c^3L^2 d.
\]

We have now bounded each term in the right-hand side of \eqref{eq:ubudeltavbis}; the dominant term is \(I_5\), with \((\E(\|I_5\|^2))^{1/2} = \mathcal{O}(h^{5/2})\). In Assumption \ref{as4} we have to split \(\phi_h(\widehat{\xi}_n,t_n)-\psi_h(\widehat{\xi}_n,t_n)\) into two parts, \(\alpha\) and \(\beta\), with \(\beta\) of higher order; for the \(v\) component we then define \(\alpha_v = I_5\) and \(\beta_v = \Delta_v -\alpha_v\). The  bounds obtained above yield
\[
\big(\E(\|\alpha_v\|^2)\big)^{1/2}\leq \frac{\sqrt{6}}{24}h^{5/2} \gamma^{1/2} c^{3/2} L d^{1/2}
\]
and
\[
\big(\E(\|\beta_v\|^2)\big)^{1/2}\leq \frac{h^3}{24} \Big( (1+2\sqrt{3})c^2L^{3/2}
+ \gamma^2 c L^{1/2} +\gamma c^{3/2} L +\sqrt{3} c^2L_1\Big) d^{1/2}.
\]

\emph{Third step.} From \eqref{eq:truex} and \eqref{eq:ubux}, the \(x\)-component of \(\phi_h(\widehat{\xi}_n,t_n)-\psi_h(\widehat{\xi}_n,t_n)\) is given by
\begin{equation}\label{eq:deltaxubu}
\Delta_x = - \int_{t_n}^{t_{n+1}} \calF(t_{n+1}-s) c \nabla f(x(s))\,ds + h  \calF(h/2) c\nabla f(y_n);
\end{equation}
again UBU replaces the integral by a midpoint approximation. By expanding the integrand by means of the fundamental theorem of calculus as
\begin{align*}
& \calF(t_{n+1}-s) c \nabla f(x(s)) = \calF(h/2) c\nabla f(x(t_{n+1/2}))\\
&\qquad +\int_{t_{n+1/2}}^s
\Big(  \calF(t_{n+1}-s^\prime) c \calH(x(s^\prime))v(s^\prime)-\calE(t_{n+1}-s^\prime) c \nabla f(x(s^\prime))\Big) ds^\prime,
\end{align*}
\eqref{eq:deltaxubu} becomes
\begin{equation}\label{eq:deltaxububis}
\Delta_ x = -hc \calF(h/2) \big(\nabla f(x(t_{n+1/2})) -\nabla f(y_n)\big)+I_6+I_7
\end{equation}
with
\begin{eqnarray*}
I_6 & = & - \int_{t_n}^{t_{n+1}} ds \int_{t_{n+1/2}}^s \calF(t_{n+1}-s^\prime) c \calH(x(s^\prime))v(s^\prime) ds^\prime,
\\
I_7 & = &  \int_{t_n}^{t_{n+1}} ds \int_{t_{n+1/2}}^s \calE(t_{n+1}-s^\prime) c \nabla f(x(s^\prime)) ds^\prime.
\end{eqnarray*}

Now, in \eqref{eq:deltaxububis}, taking \eqref{eq:ububounddeltay} into account and recalling that \(\calF(t)\leq t\),
\[
\E\Big(\|hc \calF(h/2) \big(\nabla f(x(t_{n+1/2})) -\nabla f(y_n)\big)\|^2\Big) \leq
\frac{h^8}{192} c^4L^3d.
\]

Next, for \(I_6\) (it is necessary to put absolute value bars around the inner integrals because \(s\) could be  \(<t_{n+1/2}\)),
\begin{eqnarray*}
\E\big(\|I_6\|^2\big)&\leq&c^2 \E\left[\left(\int_{t_{n}}^{t_{n+1}} ds \left|\int_{t_{n+1/2}}^s \calF(t_{n+1}-s^\prime)^2ds^\prime \right|\right)\right.\times \\
&&\qquad\qquad\qquad
\left.\left(
\int_{t_{n}}^{t_{n+1}} ds\left| \int_{t_{n+1/2}}^s  \|\calH(x(s^\prime))v(s^\prime)\|^2ds^{\prime}\right|\right)\right]\\
&\leq& c^2\frac{7h^4}{96}\times\frac{h^2}{4}L^2cd = \frac{7h^6}{384} c^3L^2d.
\end{eqnarray*}

The integral \(I_7\) may be rewritten as
\[I_7=
\int_{t_{n+1/2}}^{t_{n+1}} ds \int_{t_{n+1/2}}^s ds^\prime \int_{2t_{n+1/2}-s^\prime}^{s^{\prime}}
\frac{d}{ds^{\prime\prime}}\Big(\calE(t_{n+1}-s^{\prime\prime}) c \nabla f(x(s^{\prime\prime}))ds^{\prime\prime}\Big)ds^{\prime\prime},
\]
and, after performing the differentiation in the integrand,
\( I_7 = -\gamma^{-1} (I_1+I_2)
\), so that we may use the available bounds for \(I_1\) and \(I_2\).

Taking all the partial  bounds to \eqref{eq:deltaxububis},
\[
\big(\E(\|\Delta_x\|^2)\big)^{1/2}\leq \frac{h^3}{24}\Big(\sqrt{3}h c^2L^{3/2}+(\frac{\sqrt{42}}{2}+1) c^{3/2} L+\gamma cL^{1/2}\Big) d^{1/2}.
\]
As we see, for the \(x\) component there is no \(\mathcal{O}(h^{5/2})\) contribution and therefore we take \(\alpha_x = 0\) and \(\beta_x = \Delta_x\).

\emph{Fourth step.} With a view to checking at Step 5 condition \eqref{eq:LE2} in Assumption \ref{as4}, we estimate
\(
|\E \big( \langle \widetilde v_{n+1} -v_{n+1}, \alpha_v\rangle\big)|\);
here \(\langle \cdot,\cdot \rangle\) is the standard inner product in \(\R^d\), \(v_{n+1}\) is the velocity component of \(\xi_{n+1}\) and \(\widetilde v_{n+1}\) denotes the velocity component of a numerical step starting from \(\widehat \xi_n=(\widehat v_n,\widehat x_n)\). (This should not be confused with \(\widehat v_{n+1}\), the \(v\)-component of the random variable  \(\widehat \xi_{n+1}\) to be introduced at the next time level in the construction leading to Theorem~\ref{theo:main}.)

Since, conditional on \(\widehat v_n\), \(v_n\), the expectation of the stochastic integral \(\alpha_v = I_5\) is zero, we may write
\begin{eqnarray*}
\left|\E \big( \langle \widetilde v_{n+1} -v_{n+1}, \alpha_v\rangle\big)\right|
&=&
\left|\E \big( \langle \widetilde v_{n+1}-\widehat v_n -v_{n+1}+v_n, \alpha_v\rangle\big)\right|\\
& \leq& \Big(\E (\|\widetilde v_{n+1}-\widehat v_n -v_{n+1}+v_n\|^2)\Big)^{1/2}
\Big(\E (\|\alpha_v\|^2)\Big)^{1/2}.
\end{eqnarray*}
Now, from \eqref{eq:ubuv},
\[
\widetilde v_{n+1}-\widehat v_n -v_{n+1}+v_n = (\calE(h)-1) (\widehat v_n-v_n)-h\calE(h/2)c (\nabla f(\widetilde y_n) -\nabla  f(y_n))
\]
with (see \eqref{eq:ubuy})
\[
\widetilde  y_n  =  \widehat x_n + \calF( h/2)\widehat v_n+
\sqrt{2\gamma c} \int_{t_n}^{t_{n+1/2}}\calF(t_{n+1/2}-s) dW(s),
\]
and thus, since \(1-\calE(h) \leq \gamma h\) and \(\calE(h/2) \leq 1\),
\begin{eqnarray*}
&&\Big(\E (\|\widetilde v_{n+1}-\widehat v_n -v_{n+1}+v_n\|^2)\Big)^{1/2} \\&& \qquad\qquad\qquad\leq h\gamma
\Big(\E (\|\widehat v_n -v_n\|^2)\Big)^{1/2}+hc L \Big(\E (\|\widetilde y_n -y_n\|^2)\Big)^{1/2}.
\end{eqnarray*}
Taking into account \eqref{eq:ubuy} and the definition of \(\widetilde y_n\)
\[
\Big(\E \|\widetilde y_n -y_n\|^2\Big)^{1/2} \leq \Big(\E (\|\widehat x_n -x_n\|^2)\Big)^{1/2}+\frac{h}{2}
\Big(\E (\|\widehat v_n -v_n\|^2)\Big)^{1/2},
\]
and we conclude that \(|\E \big( \langle \widetilde v_{n+1} -v_{n+1}, \alpha_v\rangle\big)|\) is bounded above by
\begin{equation}\label{eq:mondaytwo}
h \left( \Big(\gamma+\frac{h}{2} cL\Big) \Big(\E (\|\widehat v_n -v_n\|^2)\Big)^{1/2}+cL \Big(\E (\|\widehat x_n -x_n\|^2) \Big)^{1/2}\right) \Big(\E (\|\alpha_v\|^2)\Big)^{1/2}.
\end{equation}

\emph{Fifth step.} The preceding analysis holds for all values of the parameters. We now focus in the case where \(\gamma = 2\) and \(h\leq 2\) as in the statement of Theorem~\ref{theo:ubu}. To complete the proof it is enough to translate the Euclidean norm bounds in Steps 1--4 into \(P\)-norm bounds.

To establish \eqref{eq:LE1}, we note that, because \(\alpha_x=0\),
\[
\left| \Big\langle \psi_h(\widehat{\xi}_n,t_n)- \psi_h(\xi_n,t_n),\alpha_h(\widehat{\xi}_n,t_n)\Big\rangle_{L^2,P_h}\right|=|\E \big( \langle \widetilde v_{n+1} -v_{n+1}, \alpha_v\rangle\big)|.
\]
The right hand-side of this expression was bounded in \eqref{eq:mondaytwo} in terms of \(\E (\|\widehat v_n -v_n\|^2)\), \(\E(\|\widehat x_n -x_n\|^2)\) and \(\E (\|\alpha_v\|^2)\). We now take into account \eqref{eq:sandwich} to  bound \(\E (\|\widehat v_n -v_n\|^2)\) and \(\E(\|\widehat x_n -x_n\|^2)\)
by a multiple of \(\|\widehat{\xi}_n-\xi_n \|_{L^2,P_h}\) and to bound \(\E (\|\alpha_v\|^2)\) by a multiple of \(\|\alpha_h(\widehat{\xi}_n,t_n)\|_{L^2,P_h}\).

The estimates in \eqref{eq:LE2} are established in a similar way.

{\color{black}
\subsection{The local error for UBU without extrasmoothness}
\label{sec:LEubu2}
Let us now assume that Assumptions \ref{as1}--\ref{as2} hold but Assumption \ref{as3} does not. Then the strong order of UBU is one. The bound for \(\E(\|\Delta_x\|^2)\) derived in the third step of the preceding subsection remains valid (note that it does not involve the constant \(L_1\)). However for the component \(\Delta_y\) of the local error, the Ito-Taylor expansion  leading to \eqref{eq:july17bis} cannot be taken beyond \eqref{eq:july17} because now \(d\dot\chi\) does not make sense. After replacing \(\dot\chi(s)\) by its expression in terms of \(f\), one obtains the bound
\[
\big(\E(\|\Delta_y\|^2)\big)^{1/2}\leq \frac{h^2}{4} \gamma c^{3/2} L d^{1/2}+\frac{h^2}{4} \gamma cL^{1/2}d^{1/2}+ \frac{\sqrt{3}}{12} h^3 c^2L^{3/2}d^{1/2}.
\]
Then, after combining the \(x\) and \(v\) contributions and setting \(\gamma=2\), we have the bound
\begin{eqnarray*}
\Big(\E\big(\|\phi_h(\widehat{\xi}_n,t_n)-\psi_h(\widehat{\xi}_n,t_n)\|^2 \big)\Big)^{1/2}&\leq&
\frac{h^2}{4} \left[(1+\left(\frac{1}{6}+\frac{\sqrt{42}}{12}\right)h\right]c^{3/2}Ld^{1/2}\\&&\qquad+
\frac{h^2}{2} \left(1+\frac{h}{6}\right) cL^{1/2}d^{1/2}\\&&\qquad+
\frac{\sqrt{3}}{12}h^3\left(1+\frac{h}{2}\right) c^2L^{3/2} d^{1/2}.
\end{eqnarray*}
 Recall that, for contractivity the integrator has to be operated with \(cL\) bounded above, so that the combinations \(c^2/L^{3/2}\), \(c^{3/2}L\), and \(cL^{1/2}\) are all \(\mathcal{O}(L^{-1/2})\) as \(L\) increases. The leading terms in the UBU bound in the display are \((h^2/4) (c^{3/2}Ld^{1/2})+ (h^2/2) cL^{1/2}d^{1/2}\). For comparison, we note from \eqref{eq:mondayjuly}, that for EE  the corresponding leading term in the bound is
\((\sqrt{3}/3) h^2 c^{3/2}L d^{1/2}\).
The conclusion is that, under Assumptions \ref{as1}--\ref{as2}, the properties of UBU are very close to those of EE.
}

\bigskip

{\bf Acknowledgement.} J.M.S.-S. has been supported by project
PID2019-104927GB-C21 (AEI/FEDER, UE). K.C. Z acknowledges support from a Leverhulme Research Fellowship (RF/ 2020-310),   the EPSRC grant EP/V006177/1, and the Alan Turing Institute.
%

%
\end{document}